%% file: main.tex
\documentclass[twoside]{article}

\usepackage[accepted]{aistats2022}
\usepackage[T1]{fontenc}
\usepackage{aecompl}
\usepackage[round, sort]{natbib}
\usepackage{hyperref}       % hyperlinks
\usepackage{url}            % simple URL typesetting
\usepackage{booktabs}       % professional-quality tables
\usepackage{amsfonts}       % blackboard math symbols
\usepackage{nicefrac}       % compact symbols for 1/2, etc.
\usepackage{microtype}      % microtypography
\usepackage{amsmath}
\usepackage{ amssymb }
\usepackage{xcolor}         % colors
\usepackage{ graphicx }
\usepackage{subfig}
\usepackage{graphbox}
\usepackage{amsthm}
\usepackage{multirow}

\def\AM{{\mathcal A}}

\def\MM{{\mathcal M}}

\def\PM{{\mathcal P}}
\def\SM{{\mathcal S}}

\def\EB{{\mathbb E}}

\def\T{\mathcal{T}}

\newcommand{\ie}{\textit{i}.\textit{e}.}
\newcommand{\eg}{\textit{e}.\textit{g}.}

\def\proj{\mathop{\rm Proj}}
\def\argmax{\mathop{\rm argmax}}

\newtheorem{thm}{Theorem}%[section]
\newtheorem{lem}{Lemma}%[section]
\newtheorem{remark}{Remark}%[section]

\begin{document}



\twocolumn[

\aistatstitle{Federated Reinforcement Learning\\ with Environment Heterogeneity}
\vspace{-0.2in}
\aistatsauthor{ 
Hao Jin \\ jin.hao@pku.edu.cn \\ Peking University 
\And 
Yang Peng \\ pengyang@pku.edu.cn \\ Peking University \And
Wenhao Yang \\ yangwenhaosms@pku.edu.cn \\ Peking University }
\vspace{0.1in}
\aistatsauthor{ 
Shusen Wang \\ shusenwang@xiaohongshu.com \\ Xiaohongshu Inc. 
\And 
Zhihua Zhang \\ zhzhang@math.pku.edu.cn \\ Peking University }
\vspace{0.1in}
]

\begin{abstract}
    We study a Federated Reinforcement Learning (FedRL) problem in which $n$ agents collaboratively learn a single policy without sharing the trajectories they collected during agent-environment interaction.
    % Compared with the existing work on FedRL, we stress the difference between state-transitions in different environments.
    %Different from the existing work on FedRL, 
    %In this paper, 
    We stress the constraint of environment heterogeneity, which means $n$ environments corresponding to these $n$ agents have different state transitions.
    To obtain a value function or a policy function which optimizes the overall performance in all environments, we propose two federated RL algorithms, \texttt{QAvg} and \texttt{PAvg}.
    We theoretically prove that these algorithms converge to suboptimal solutions, while such suboptimality depends on how heterogeneous these $n$ environments are.
    Moreover, we propose a heuristic that achieves personalization by embedding the $n$ environments into $n$ vectors.
    The personalization heuristic not only improves the training but also allows for better generalization to new environments.
\end{abstract}

\input{intro}
\input{related}
\input{method}

\input{exper}

\bibliographystyle{plainnat}
\bibliography{references}

\onecolumn
% \appendix
\input{appendix}

\end{document}

%% file: intro.tex
\section{Introduction}

In recent years, reinforcement learning (RL) \cite{sutton1998introduction} has made unprecedented progresses in solving challenging problems such as playing Go game \cite{hessel2018rainbow,silver2016mastering,silver2017mastering} and controlling robots \cite{fan2018surreal,levine2016end}.
Traditionally, when handling such problems, one typically assumes that the environment has a fixed state transition.
However, in some real-life applications, an agent is expected to simultaneously deal with different state transitions in multiple environments.
For example, a drone is expected to perform well under different weather conditions of the physical environment (\eg, wind speed and wind direction), which may affect the state transition.
In this way, the learning of the drone policy falls beyond the traditional assumption mentioned earlier.

In this paper, $n$ agents are assumed to be located in $n$ environments which have the same state space $\SM$, action space $\AM$, reward function $R$, but different state transitions $\{\PM_i\}_{i=1}^n$.
After incorporating environment heterogeneity into FedRL, we are mainly concerned with the following two problems.
First, it is natural to ask how to learn a single policy performing uniformly well in these $n$ environments \cite{killian2017robust,doshi2016hidden}.
However, any single policy is inevitably suboptimal compared with the optimal policy in each environment because of the environment heterogeneity.
Second, we wonder how to additionally develop a \textit{personalized} policy in each environment, which is better than the globally learned policy.
To address these issues, collaboration among these $n$ agents is necessary: interaction with any single environment is limited in diversity to learn for all $n$ environments; samples from each individual environment are also limited in quantity to learn a locally optimal policy.
Therefore, it is important to figure out how to achieve collaboration among $n$ agents in the setting of FedRL when deriving efficient solutions to these two issues.
It is worth noting that we additionally do not allow agents to communicate their interactions with individual environments in order to protect privacy embedded in their local experiences.

% Effectiveness of these approaches can be well justified within the traditional framework, where the agent is motivated to maximize its cumulative rewards in a given environment.
% However, such a theoretical framework is always challenged in real-life applications: target policy is expected to simultaneously perform well in more than one environment \cite{killian2017robust,doshi2016hidden}; data collection is distributed \cite{lowe2017multi} and communication is constrained \cite{sakuma2008privacy,wang2019privacy} because of privacy issues \cite{abe2004cross,cogill2006approximate}.
% These factors add great difficulties in deriving efficient algorithms with theoretical justification.

The setting of FedRL with environment heterogeneity is common in real-life applications.
Smart home devices are deployed in families with different using preference and habits, while service providers are interested in how to provide better experience via improving the policy loaded in these devices.
Viewing the policy as a RL agent, users with different using habits can be regarded as environments with different state transitions, which means they may response differently even to the same action.
Moreover, data collected in any certain device is usually not enough in the application to independently learn a reliable policy, while  images and audios collected by each device are sometimes inaccessible for the service providers out of privacy issues. 
In this way, the policy training for these smart devices fits into the framework of FedRL with environment heterogeneity, and the second problem of personalization in our setting perfectly describes the dilemma of service providers in improving performance of different users without accessing their data.
% Collaboration among agents in our setting is referred to sharing their models instead of communicating their experience (\ie, their observations, actions, or rewards) because of privacy issues in certain real-life applications.
% A feasible approach is letting the $n$ agents share models or part of the models via communication.
% In this way, all the experience collected by the $n$ agents are used for learning one policy or $n$ correlated policies.
% This approach known as \textit{federated reinforcement learning (FedRL)}. 
% Note that our work is substantially different from the existing FedRL methods, e.g., \cite{zhuo2019federated}.
% They assume the $n$ environments have the same state-transitions, whereas we do not make such a strong assumption.
% Furthermore, in FedRL, the agents share models but do not share 
% The motivation is to protect privacy.
% Consider such an application.
% We regard robots or smart home devices as agents that are deployed into people's homes.
% The goal is to improve the policy so that the robots or smart home devices will provide users with better experience.
% The persons they interact with can be regarded as environments.
% Since different people have different behaviors, the environments have different state-transition.
% This justifies our setting that the $n$ environments are different.
% The data collected by a single agent are not enough to train a policy, and collaborative RL is therefore necessary.
% What the agents observe, e.g., images and audios, are users' privacy and should not be shared.

To learn a uniformly good policy, we follow the approach of letting the agents share their models and propose two model-free algorithms, \texttt{QAvg} and \texttt{PAvg}.
These algorithms iteratively perform local updates on the agent side and global aggregation on the server side.
Different from the extant work in FedRL, we emphasize the role of environment heterogeneity and theoretically analyze effectiveness of these algorithms.
Our theories show that both \texttt{QAvg} and \texttt{PAvg} converge to a suboptimal solution and the suboptimality is affected by the degree of environment heterogeneity in FedRL.
Based on theoretical effectiveness of \texttt{QAvg} and \texttt{PAvg}, we also derive \texttt{DQNAvg} and \texttt{DDPGAvg} as extensions of methods with Q networks and policy networks, \ie, \texttt{DQN} and \texttt{DDPG}.
Moreover, we carry out numerical experiments on several tabular environments to verify theoretical results of \texttt{QAvg} and \texttt{PAvg}, and compare  \texttt{DQNAvg}  (\texttt{DDPGAvg}) with  \texttt{DQN} (\texttt{DDPG}) in harder tasks of control.
% To solve the first problem, we propose two model-free algorithms \texttt{QAvg} and \texttt{PAvg}, which are federated versions of Q-learning and policy gradient.
% \texttt{QAvg} and \texttt{PAvg} iteratively perform local updates on the agents side and global aggregation on the server side.
% We theoretically analyze the convergence of tabular \texttt{QAvg} and \texttt{PAvg} and show that they converge to suboptimal.
% Our theories show that the heterogeneity in the $n$ environments affects the performance.
% We empirically evaluate \texttt{QAvg} and \texttt{PAvg} under tabular setting and deep RL setting.

To  achieve personalization in different environments, we propose a heuristic with slight modification to structures of \texttt{DQNAvg} and \texttt{DDPGAvg}.
Specifically, we embed each environment into a low-dimension vector to capture its specific state transition.
During the training of FedRL, these $n$ agents periodically aggregate their parameters except their embedding layers.
Along with the learned aggregated network, the private embedding layer enables each agent to achieve better performance in its individual environment.
Such personalization heuristic also enables us to generalize the learned model in FedRL to any novel environment.
Instead of updating all parameters of the model, we only need to adjust the low-dimension embedding layer for the novel environment.
Empirical experiments have shown that our proposed heuristic not only improves training performance of the learned models in \texttt{DQNAvg} and \texttt{DDPGAvg}, but also helps to achieve stable generalization within few updates in the novel environment.
% To be specific, we embed each environment into a vector and learn $n$ policy networks.
% The policy networks, except the embedding layers, are shared.
% When the learned policy is deployed into a new environment, we need only to learn the embedding of the environment.
% Since the embedding has a small number of parameters, the learning require only a small number of parameters.

% \red{Remove this paragraph.}
% Furthermore, we propose heuristics on how to achieve both personalization in tasks of FedRL and generalize the global model to a new environment within limited steps of training.
% To characterize different environments involved in FedRL, we utilize embedding technique \cite{grbovic2018real,lee2016personalizing} \red{perhaps the citation here is bad for us. It sounds like we simply applied the existing methods.} to model them as low-dimension vectors and incorporate them into state observations.
% With assistance of such environment embeddings, the global model achieves personalization in tasks of FedRL.
% Since the global model trained in this way is believed to capture environment heterogeneity, we only have to learn corresponding embeddings when confronted with a novel environment.
% The introduction of environment embeddings has been empirically shown to stabilize generalization and outperform the methods of adjusting the entire policy.

In summary, this paper offers the following main contributions:
\begin{itemize}
    \item We propose \texttt{QAvg} and \texttt{PAvg} to solve the task of federated reinforcement learning (FedRL) with environment heterogeneity, where environments have different state transitions.
    \item We theoretically analyze the convergence of \texttt{QAvg} and \texttt{PAvg}, discuss relations between their convergent performance and environment heterogeneity in FedRL, and extend the averaging strategy to derive \texttt{DQNAvg} and \texttt{DDPGAvg} for more complicated environments.
    \item We propose a heuristic idea to achieve personalization in FedRL, which utilizes embedding layers to capture the specific state transition in individual environment.
    We have also empirically shown that such heuristic helps to generalize the learned model in FedRL to new environments in a stable and easy way.
\end{itemize}

%% file: related.tex
\section{Related Work}
\paragraph{Classical RL methods.}
Traditionally, reinforcement learning (RL) assumes the environment has a fixed state transition and seeks to maximize the cumulative rewards in the environment \cite{sutton1998introduction,watkins1992q}.
The environment is usually modelled as a standard MDP, $\MM=\langle \SM,\AM,R,\PM,\gamma\rangle$ \cite{bellman1957dynamic,bertsekas1995dynamic}.
The objective function is formulated as
\begin{equation*}
\begin{aligned}
    g_{d_0}(\pi)=\mathbb{E} \bigg[\sum_{t=1}^\infty\gamma^t R(s_t,a_t) &\; \bigg| \; s_0\sim d_0, a_t\sim\pi(\cdot|s_t),\\ &s_{t+1}\sim\PM_i(\cdot|s_t,a_t) \bigg],
\end{aligned}
\end{equation*} 
where $d_0$ represents the initial state distribution.
To solve the problem, there are many model-free methods such as Q-learning \cite{watkins1992q} and policy gradient (PG) \cite{sutton1999policy}.
Under the setting of a standard MDP, prior works \cite{sutton1998introduction,agarwal2019theory} have proved their convergence to the optimal policy.
%\cite{sutton1998introduction} showed that the Bellman update of Q-Learning corresponds a contracting mapping whose fixed point satisfies the Bellman equation of optimality.
%\cite{agarwal2019theory} analyzed the convergence of policy gradient methods.

% \paragraph{Multi-task reinforcement learning.}
% \red{This paragraph needs major revision. Directly state what MTRL is and how it is different from the standard MDP and our work.}
% Multi-Task Reinforcement Learning (MTRL) covers a wide range of RL problems with multiple different environments.
% MTRL has variours objectives, \eg, to find a policy mastering a group of diverse tasks \cite{hessel2019multi}, to improve generalization of learned policy beyond training environments \cite{oh2017zero,parisotto2015actor}, to stabilize and accelerate policy learning in different tasks via communication \cite{teh2017distral,mnih2016asynchronous}, \etc
% Various technologies have been proposed to design efficient algorithms in MTRL, \eg, introducing a regularization term w.r.t.\ the distilled policy \cite{teh2017distral}, and utilizing the off-policy correction methods \cite{espeholt2018impala}.
% While the formulation of MTRL is very useful in practice, difference between environments is hard to formalized.
% It is therefore difficult to theoretically analyze algorithms in the context of MTRL.

\paragraph{HiP-MDP and MTRL.}
FedRL is closely related to Hidden Parameter Markov Decision Processes (HiP-MDP) \cite{doshi2016hidden,killian2017robust} and Multi-Task Reinforcement Learning (MTRL) \cite{teh2017distral,espeholt2018impala,mnih2016asynchronous}.
HiP-MDP assumes the existence of latent variables which decide the state transition of an environment.
In \cite{doshi2016hidden,killian2017robust} it explicitly learns the natural distribution of latent variables with a generative network and considers Bayesian reinforcement learning.
FedRL is similar to HiP-MDP when talking about the source of environment heterogeneity, but it additionally has constraints on privacy issues, which does not allow agents to share their collected experiences.
MTRL assumes that the $n$ agents located in different environments are different and they perform different tasks.
FedRL can be viewed as a special case of MTRL where the agents perform the same task.
While in \cite{zeng2020decentralized} it also concentrates on methods of policy averaging, our work additionally focuses on the specific personalization problem in the federated setting.

% To clarify the setting of FedRL with environment heterogeneity, we have to discuss its relations with other RL problems with either different environments or multiple agents.
% Based on the assumption that there exist latent variables deciding state-transitions of the environment, HiP-MDP and its variants \cite{doshi2016hidden,killian2017robust} explicitly learn the natural distribution of latent variables with a generative network and apply Bayesian reinforcement learning.
% By contrast, FedRL with environment heterogeneity faces many practical constraints, such as constraints on experience sharing, which remarkably distinguishes our work from the HiP-MDPs.
% Multi-Task Reinforcement Learning (MTRL) covers a wide range of RL problems involving multiple agents in multiple environments with different tasks \cite{teh2017distral,espeholt2018impala,mnih2016asynchronous}.
% FedRL with environment heterogeneity can be viewed as a special case of MTRL, while tasks of different agents in FedRL are rather similar with each other.
% Yet the specific formulation of FedRL enables us to theoretically analyze the proposed algorithms, which is lacked for most MTRL methods.
% Moreover, the personalized extension of our algorithms utilizes the formulation of FedRL to achieve fast and stable generalization. (Stress the importance of FedRL specific formulation).

\paragraph{Federated Learning.}
FL, also known as federated optimization, allows local devices to collaboratively train a model without data sharing \cite{mahajan2018efficient}.
To reduce the communication cost in FL, many communication-efficient algorithms have been proposed, \eg, \texttt{FedAvg} \cite{mahajan2018efficient} and \texttt{FedProx} \cite{sahu2018federated}.
The communication-efficient FL algorithms let each client locally update the model using its local data and periodically aggregate the local models.
Our proposed algorithms bear a resemblance with the FL algorithms: an agent performs multiple local updates between two communications.
Similar methods have also been previously studied in contexts of FedRL \cite{liu2019lifelong,zhuo2019federated,wang2020federated,nadiger2019federated}:
\cite{wang2020federated} simply analyzes the convergence speed of policy gradient in FedRL tasks without considering environment heterogeneity while \cite{liu2019lifelong,zhuo2019federated,nadiger2019federated} mainly concentrates on applications in specific scenarios.
Moreover, \cite{nadiger2019federated} considers personalization of FedRL in a specific application.

\paragraph{Personalized Federated Learning.}
FL seeks to learn a single model that performs uniformly well on all $n$ local datasets, while personalized FL aims to learn $n$ models specialized for the $n$ local datasets.
Many personalized FL methods have been developed:
\cite{arivazhagan2019federated} designed a neural network architecture with personalization layers which are not shared;
\cite{mansour2020three,deng2020adaptive} viewed the global model as the interpolation of local models;
\cite{bui2019federated} introduced a technical called private embedding.
In this paper, we extend personalization in federated learning to the context of FedRL, which means we aim to additionally learn $n$ different policies for each environment with the collaboration among agents.

%% file: method.tex
\section{Federated Reinforcement Learning}

Suppose $n$ agents respectively interact with $n$ independent environments.
The environments have different state transitions $\{\PM_i\}_{i=1}^n$ but the same state space $\SM$, action space $\AM$, and reward function $R$.
These environments are modelled as Markov Decision Processes (MDPs), $\MM_i=\langle \SM,\AM,R,\PM_i,\gamma\rangle$, for $i = 1, \cdots , n$.

The goal of Federated Reinforcement Learning (FedRL) is letting the $n$ agents jointly learn a policy function or a value function that performs uniformly well across the $n$ environments.
Due to privacy constraints, the $n$ agents cannot share their collected experience.
Policy-based FedRL can be formulated as the following optimization problem:
\begin{equation}
\begin{aligned}
    \label{eq:gd0}
    \max_{\pi} \;
    \Bigg\{
    g_{d_0}(\pi) &\; \triangleq \; \frac{1}{n}\sum_{i=1}^n\EB \Bigg [ \sum_{t=1}^\infty\gamma^t R(s_t,a_t)|s_0\sim d_0,\\ & a_t\sim\pi(\cdot|s_t), s_{t+1}\sim\PM_i(\cdot|s_t,a_t) \Bigg]
    \Bigg\} ,
\end{aligned}
\end{equation}
where $d_0$ represents the common initial state distribution in these $n$ environments.
If state transitions $\{\PM_i\}_{i=1}^n$ are the same, the optimal policy $\pi^*$ is independent of $d_0$ \cite{bellman1959functional}.
However, if these state transitions are different, then the solution to Eq.~\eqref{eq:gd0} actually depends on $d_0$.

\begin{thm}
\label{fedrl-example}
There exists a task of FedRL with the following properties.
Assume that $\pi^\star \in \argmax_{\pi} g_{d_0} (\pi)$.
There exist another initial state distribution $d_0'$ and another policy $\tilde{\pi}$ such that $g_{d_0}(\tilde{\pi})<g_{d_0}(\pi^\star)$, but $g_{d_0'} (\tilde{\pi}) > g_{d_0'} (\pi^\star ) $. 
\end{thm}

Theorem \ref{fedrl-example} shows that there does not exist an optimal policy $\pi^\star$ that dominates all policies for all $d_0$.
We denote the solution to \eqref{eq:gd0} by $\pi^\star_{d_0}$ which means the initial state distribution affects the optimal policy.

\section{Algorithms: \texttt{QAvg} and \texttt{PAvg}}

We propose two novel FedRL algorithms, \texttt{QAvg} and \texttt{PAvg}, for learning a value function and a policy function, respectively.
We discuss tabular versions of \texttt{QAvg} and \texttt{PAvg}; versions of neural networks, such as \texttt{DQNAvg} and \texttt{DDPGAvg}, can be similarly implemented.
These algorithms alternate between local computation and global aggregation.
Specifically, each agent locally updates its value function or policy function for multiple times, and then the server averages these $n$ functions of all agents.
To improve the communication efficiency, the local updates are performed multiple times between two communications.

\texttt{QAvg} learns an $|\SM|\times|\AM|$ table by alternating between local updates and global aggregations.
For $k= 1 , \cdots , n$, the $k$-th agent performs the following local update:
\begin{equation*}
\begin{aligned}
    Q^k_{t+1} \big(s, a \big)
    \; \leftarrow \;
    \big(1 {-} \eta_t \big) \cdot Q^k_t \big( s, a \big)  + \,
    \eta_t \cdot \Big[ R \big(s, a \big) \\
    + \gamma \sum_{s'} \PM_k(s'|s,a) \max_{a' \in \AM} Q_t^k \big( s' , a' \big) \Big].
\end{aligned}
\end{equation*}
In the equation, the superscript $k$ indexes the environment $\MM_k$, and the subscript $t$ indexes the iteration.
After several local updates, there is a global aggregation:
\begin{align*}
    \bar{Q}_t(s,a) &\leftarrow \frac{1}{n}\sum_{i=1}^n Q_t^i(s,a), \; \forall \, s,a;
    \\
    Q_t^i(s,a) &\leftarrow \bar{Q}_t(s,a), \; \forall \, s,a,k .
\end{align*}
Throughout, only Q tables are communicated, while agents do not share their collected experience.

\texttt{PAvg} seeks to learn a $|\SM|\times|\AM|$ table, $\bar{\pi} (a|s) $.
Each agent independently repeats the local update for multiple times:
\begin{align*}
    \tilde{\pi}_{t+1}^k(a|s)&\leftarrow \pi_t^k(a|s)+\frac{\partial g_{d_0,k}(\pi_t^k)}{\partial \pi(a|s)}, \; \forall \, s,a,k;\\
    \pi_{t+1}^k(\cdot|s)&\leftarrow \proj_{\Delta(\AM)}(\tilde{\pi}_{t+1}^k(\cdot|s)), \; \forall  \, s,a,k.
\end{align*}
Here, $g_{d_0,k}(\pi)=\EB[\sum_{t=1}^\infty\gamma^t R(s_t,a_t)|s_0\sim d_0, a_t\sim\pi(\cdot|s_t), s_{t+1}\sim\PM_k(\cdot|s_t,a_t)]$ is the $k$-th agent's objective function, and $\proj_{\Delta(\AM)}$ is the projector onto the simplex of action space $\Delta(\AM)$.
Then, there is a global aggregation after several local updates:
\begin{align*}
    \bar{\pi}_t(a|s) &\leftarrow \frac{1}{n}\sum_{i=1}^n \pi_t^i(a|s), \; \forall \, s,a;
    \\
    \pi_t^i(a|s) &\leftarrow \bar{\pi}_t(a|s), \; \forall \, s,a,k .
\end{align*}
Similar to \texttt{QAvg}, agents in \texttt{PAvg} only share their policy functions throughout the training process.

\section{Theoretical Analyses}
\label{sec:theory}
In this section we prove that both \texttt{QAvg} and \texttt{PAvg} converge to suboptima whose performance across the $n$ environments are theoretically guaranteed.
We also discuss how the suboptimality of convergent policies is affected by the environment heterogeneity in FedRL.

\subsection{Notation}
\label{theo_pre}

% Although the objective function of FedRL is the sum of objectives in the $n$ traditional RL problems, the optimality theory in FedRL is quite different from those in the traditional RL problem with a fixed environment (see Theorem \ref{fedrl-example}). 
% Therefore, we are next to introduce some properties about FedRL for better comprehension of the following theoretical analysis.

\textbf{Imaginary environment $\MM_I$.} 
Let $\PM_1, \cdots , \PM_n$ be the state transition functions of the $n$ environments.
Define the average state transition:
\begin{equation*}
    \bar{\PM}(s'|s,a)=\frac{1}{n}\sum_{k=1}^n\PM_k(s'|s,a),~\forall s,s'\in\SM,~\forall a\in\AM.
\end{equation*}
To analyze the convergence of proposed algorithms, we introduce the imaginary environment, $\MM_I=\langle\SM,\AM,R,\bar{\PM},\gamma\rangle$.
As its name suggests, the imaginary environment $\MM_I$ does not have to be one of the $n$ environments in FedRL, \ie, $\MM_I\notin \{\MM_i\}_{i=1}^n$.

\textbf{Environment heterogeneity.}
In FedRL, different environments $\{\MM_i\}_{i=1}^n$ have different state transitions $\{\PM_i\}_{i=1}^n$.
Intuitively speaking, the closer these state-transitions are, the easier the problem.
To quantify the environment heterogeneity, we define
\begin{align*}
    \kappa_1
    & \triangleq 
    \max_{s, \pi } \sum_{s'}  \sum_{i=1}^n   \bigg|  \PM_i^\pi(s'|s)-\frac{1}{n}\sum_{j=1}^n\PM_j^\pi(s'|s)  \bigg| , \\
    \kappa_2
    & \triangleq 
    \max_{\pi} \frac{1}{n} \sum_{i=1}^{n} \bigg\|\nabla_\pi g_{d_0,i}(\pi)-\frac{1}{n}\sum_{j=1}^{n}\nabla_\pi g_{d_0,j}(\pi) \bigg\|_{2} ,
\end{align*}
where $\PM_k^\pi(s'|s)\triangleq\EB_{A\sim\pi(\cdot|s)}\left[\PM_k(s'|s,A)\right]$.
If the state transitions in FedRL are close to each other, both $\kappa_1$ and $\kappa_2$ are small.

\subsection{Theoretical Analysis of QAvg}
\texttt{QAvg} is the federated version of Q-Learning.
Traditional analysis of Q-learning claims that Q-learning converges to the Q function of optimal policy in that given environment.
Similarly, theoretical analysis of \texttt{QAvg} mainly focuses on convergence performance of the averaged Q function, \ie, $~\bar{Q}_t$ shown in its aggregation at time $t$.

To better understand the convergence of \texttt{QAvg}, we return to the definition of $g_{d_0}(\pi)$.
The objective of FedRL is decomposed as follows:
\begin{equation*}
    g_{d_0}(\pi)=\frac{1}{n}\EB_{S_0\sim d_0}\left[\sum_{i=1}^n V_i^\pi(S_0)\right]=\EB_{S_0\sim d_0}\left[\bar{V}^\pi(S_0)\right],
\end{equation*}
where $\{V_i^\pi\}_{i=1}^n$ are normally defined value functions of policy $\pi$ in the $n$ environments $\{\MM_i\}_{i=1}^n$:
\begin{align*}
    V^\pi_i(s)=\EB_{\pi,\PM_i}\Bigg[\sum_{t=0}^\infty\gamma^tR(s_t,a_t)|s_0=s,a_t\sim\pi(\cdot|s_t)\Bigg],
\end{align*}
and $\bar{V}^\pi$ is the averaged value function $\bar{V}^\pi=\frac{1}{n}\sum_{i=1}^nV_i^\pi$.
The dependence of optimality in FedRL with the initial state distribution $d_0$ indicates that $\bar{V}^\pi$ is not the value function of any environment (otherwise, there exists optimal policy $\pi^*$ independent with $d_0$).

Imaginary environment $\MM_I$ is therefore introduced to element-wisely lower bound the values of $\bar{V}^\pi$.
Specifically, the value function $V^\pi_I$ of the policy $\pi$ in the imaginary environment $\MM_I$ is normally defined as follows:
\begin{equation*}
    V^\pi_I(s)=\EB_{\pi,\bar{\PM}}\left[\sum_{t=0}^\infty\gamma^tR(s_t,a_t)|s_0=s,a_t\sim\pi(\cdot|s_t)\right],
\end{equation*}
and its relationship with the averaged value function $\bar{V}^\pi$ is mainly described in Lemmas \ref{lem:v:lowerbound} and \ref{lem:v:bound}.
As long as the environments $\{\MM_i\}_{i=1}^n$ are not too different from each other, the value function $V^\pi_I$ manages to properly approximate $\bar{V}^\pi$.
\begin{lem} \label{lem:v:lowerbound}
For all state $s$ and policy $\pi$, we have $\bar{V}^\pi (s) \geq V_{I}^\pi (s)$.
\end{lem}
\begin{lem} \label{lem:v:bound}
Let $\kappa_1$ be the environment heterogeneity.
For all $s$ and $\pi$,  we have
\begin{align*}
    \Big| \bar{V}^\pi \big( s \big) \, - \, V^\pi_I \big( s \big) \Big|
    \; \leq \;
    \frac{\gamma\kappa_1}{(1-\gamma)^2}.
\end{align*}
\end{lem}
After identifying $V^\pi_I$ as a lower bound of $\bar{V}^\pi$, it is natural to consider the optimal policy $\pi^*_I$ in the imaginary environment $\MM_I$.
Because of its optimality in $\MM_I$, the value function $V_I^{\pi^*_I}$ dominates the value function of any other policy $\pi$, \ie, $~V_I^{\pi^*_I}(s)\geq V_I^{\pi}(s),\forall s$.
In other words, $\pi^*_I$ reaches the largest lower bound of the averaged value function $\bar{V}^\pi$.

\begin{figure*}[t]
    \centering
    \includegraphics[width=0.98\linewidth]{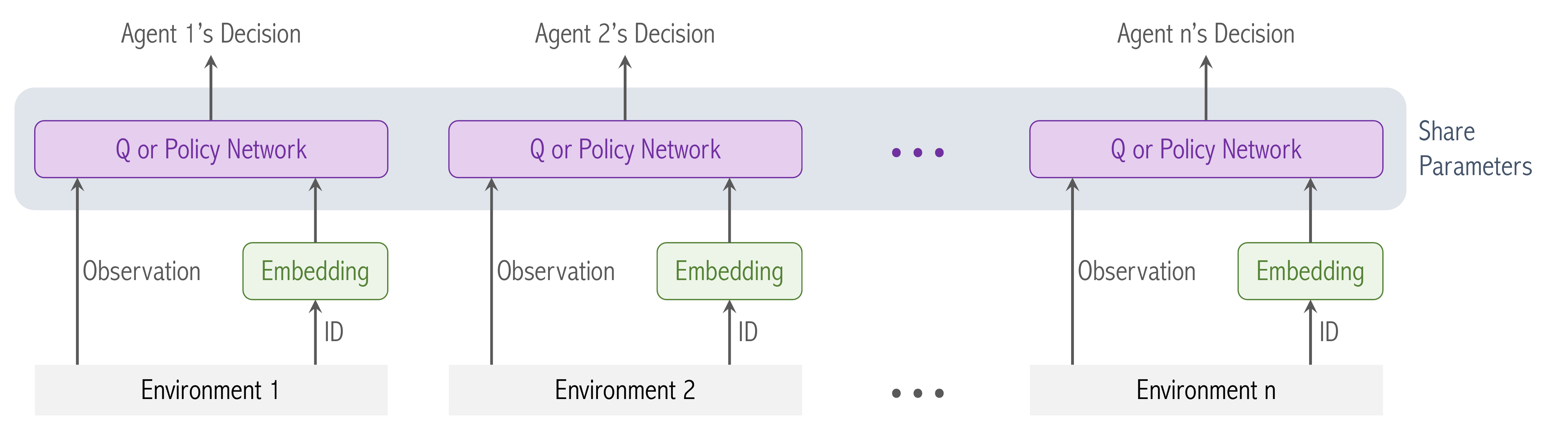}
    \caption{The figure shows the Q or policy networks of personalized FedRL. The networks, except the embedding layers, share parameters. 
    To generalize the model to new environments, we keep the trained Q or policy networks but train the embedding layer from random initialization.}
    \label{fig:peravg}
\end{figure*}

Finally, we are ready to show the convergence results of \texttt{QAvg}.
Taking the number of local updates as $E$, the algorithm with $E\geq 1$ not only converges, but also reaches the Q function of $\pi^*_I$ in $\MM_I$:
\begin{thm}[Convergence results of \texttt{QAvg}]
\label{qavg-converge}
    Take $\bar{Q}_t$ as the average of distributed Q functions $Q_t^k$ in the $n$ environments at iteration $t$, \ie, $\bar{Q}_t=\frac{1}{n}\sum_{k=1}^n Q_t^k$.
    Let the number of local updates be $E$.
    Assume $Q^{\pi_I^*}_I$ is the Q function of optimal policy $\pi^*_I$ in $\MM_I$.
    Letting $\eta_t=\frac{2}{(1-\gamma)(t+E)}$,  we have 
    \begin{align*}
        \Big\| \bar{Q}_t-Q^{\pi_I^*}_I \Big\|_\infty
        \; \le \;
        \frac{16\gamma E}{(1-\gamma)^3(t+E)}.
    \end{align*}
\end{thm}
\begin{remark}
$E=1$ makes a special variant of \texttt{QAvg}.
This means the agents communicate after every local update of their Q functions.
Although the heavy communication load makes \texttt{QAvg} with $E=1$ quite impractical, it provides intuitions on how \texttt{QAvg} achieves the optimal Q function of $\pi_I^*$.
The update of every local Q function in \texttt{QAvg} with $E=1$ is formulated as follows:

\begin{small}
\begin{align*}
    Q_{t+1}^j(s,a)&\leftarrow \frac{1}{n}\sum_{k=1}^n\Big[ R(s,a)+\gamma\sum_{s'}\PM_k(s'|s,a)\max_{a'}Q_t^k(s',a')\Big]\\
    & \quad =R(s,a)+\gamma \sum_{s'}\bar{\PM}(s'|s,a)\max_{a'}Q_t^j(s',a'),
\end{align*}
\end{small}

\noindent where the last equality is because the local Q functions keep the same in \texttt{QAvg} with $E=1$.
In this way, every local Q function is updated as if the agent were trained in the imaginary environment $\MM_I$.
\end{remark}
\begin{remark}
\texttt{QAvg} with $E=\infty$ corresponds to the algorithm which never communicates and simply averages those independently trained Q functions as the aggregated Q function.
Neither its theoretical convergence nor its empirical performance is similar to that of \texttt{QAvg} with $E<\infty$.
\end{remark}

\subsection{Theoretical Analysis of PAvg}
\texttt{PAvg} directly views the policy $\pi$ as optimization parameters in maximizing the objective function.
In this way, the corresponding theoretical analysis focuses on the convergence of objective values $g_{d_0}(\bar{\pi}_t)$, where $\bar{\pi}_t$ represents the averaged policy shown in aggregation of \texttt{PAvg} at time $t$.
\begin{thm}[Convergence performance of \texttt{PAvg}]
\label{pavg-converge}
Denote $L$ as the L-smoothness parameter of $g_{d_0}(\pi)$ w.r.t.\ $\pi$, 
 $E$ as the number of local updates,
and $\kappa_2$ as the environment heterogeneity.
Letting $\eta_t = \sqrt{\frac{E}{12L^2(t+E/3)}}$,  we have that
\begin{align*}
     \max_{t=0,..,T-1} g_{d_0}(\bar{\pi}_t)
     \; \geq \; g_{d_0}(\pi_{d_0}^\star) - c \cdot \Big(\kappa_2+\frac{1}{\sqrt{T}} \Big),
\end{align*}
where $c$ is constants and logarithmic factors of T.
\end{thm}
\begin{remark}
\label{pavgE}
Here we discuss the effect of local iterations $E$ on the convergence.
The term $c \cdot (\kappa_2+\frac{1}{\sqrt{T}} )$ in the theorem is equal to
\begin{equation*}
    C_1 + T^{-0.5} \cdot \big(C_2 E^{-0.5} + C_3 E^{0.5} + C_4 E^{2.5} \big) .
\end{equation*}
Here, $C_1, C_2, C_3, C_4$ are either independent of $T$ and $E$ or contain logarithmic factors of $T$ and $E$,
implying  there exists an $E$ that is the best for the convergence.
\end{remark}

\begin{table*}[t]
    \centering
    \begin{tabular}{c|ccc|ccc}
    \multirow{2}{*}{} & \multicolumn{3}{c|}{RandomMDPs} & \multicolumn{3}{c}{WindyCliffs}\\ \cline{2-7}
         & QAvg  & SoftPAvg  & ProjPAvg & QAvg  & SoftPAvg  & ProjPAvg \\ \hline
        $\kappa=0$ & 35.42\begin{tiny}$\pm 0.05$\end{tiny} & 35.15\begin{tiny}$\pm 0.05$\end{tiny} & 34.97\begin{tiny}$\pm 0.05$\end{tiny} & 133.97\begin{tiny}$\pm 0.00$\end{tiny} & 133.97\begin{tiny}$\pm 0.00$\end{tiny} & 119.97\begin{tiny}$\pm 0.34$\end{tiny}\\ \hline
        $\kappa=0.2$ & 35.23\begin{tiny}$\pm 0.05$\end{tiny} & 34.97\begin{tiny}$\pm 0.05$\end{tiny} & 34.92\begin{tiny}$\pm 0.05$\end{tiny} & 133.97\begin{tiny}$\pm 0.00$\end{tiny} & 133.97\begin{tiny}$\pm 0.00$\end{tiny} & 118.47\begin{tiny}$\pm 0.34$\end{tiny}\\ \hline
        $\kappa=0.4$ & 34.80\begin{tiny}$\pm 0.05$\end{tiny} & 34.58\begin{tiny}$\pm 0.05$\end{tiny} & 34.54\begin{tiny}$\pm 0.05$\end{tiny} & 133.97\begin{tiny}$\pm 0.00$\end{tiny} & 133.96\begin{tiny}$\pm 0.00$\end{tiny} & 115.82\begin{tiny}$\pm 0.34$\end{tiny}\\ \hline
        $\kappa=0.6$ & 34.14\begin{tiny}$\pm 0.06$\end{tiny} & 34.02\begin{tiny}$\pm 0.06$\end{tiny} & 34.02\begin{tiny}$\pm 0.06$\end{tiny} & 133.96\begin{tiny}$\pm 0.00$\end{tiny} & 133.95\begin{tiny}$\pm 0.00$\end{tiny} & 111.46\begin{tiny}$\pm 0.35$\end{tiny}\\ \hline
        $\kappa=0.8$ & 33.29\begin{tiny}$\pm 0.06$\end{tiny} & 33.25\begin{tiny}$\pm 0.06$\end{tiny} & 33.38\begin{tiny}$\pm 0.06$\end{tiny} & 133.65\begin{tiny}$\pm 0.03$\end{tiny} & 133.59\begin{tiny}$\pm 0.03$\end{tiny} & 103.53\begin{tiny}$\pm 0.36$\end{tiny}
    \end{tabular}
    \caption{Impact of environment heterogeneity on convergent performance: larger $\kappa$ indicates environments with larger environment heterogeneity, \ie, $\{\PM_k^\kappa\}_{k=1}^N$ with larger noise from $\PM_0$; \texttt{QAvg} ($E=4$), \texttt{SoftPAvg} ($E=4$) and \texttt{ProjPAvg} ($E=32$) are evaluated on the noiseless environment $\PM_0$; each setting is repeated with $16,000$ random seeds, and we display the mean with standard error.}
    \label{tab:env_heter}
\end{table*}

\begin{figure*}[t]
    \centering
    \begin{minipage}{.42\linewidth}
        \begin{tabular}{c|ccc}
        \multirow{2}{*}{} &  \multicolumn{3}{c}{WindyCliffs}\\ \cline{2-4}
             & QAvg  & SoftPAvg  & ProjPAvg \\ \hline
            E=1 & 129.55\begin{tiny}$\pm 0.17$\end{tiny} & 126.92\begin{tiny}$\pm 0.19$\end{tiny}& 122.08\begin{tiny}$\pm 0.30$\end{tiny}\\ \hline
            E=2 & 129.55\begin{tiny}$\pm 0.17$\end{tiny}& 129.56\begin{tiny}$\pm 0.17$\end{tiny}& 123.28\begin{tiny}$\pm 0.29$\end{tiny}\\ \hline
            E=4 & 129.55\begin{tiny}$\pm 0.17$\end{tiny}& \textbf{129.65}\begin{tiny}$\pm 0.17$\end{tiny}& 124.94\begin{tiny}$\pm 0.27$\end{tiny}\\ \hline
            E=8 & 129.55\begin{tiny}$\pm 0.17$\end{tiny}& 129.62\begin{tiny}$\pm 0.17$\end{tiny}& \textbf{126.03}\begin{tiny}$\pm 0.25$\end{tiny}\\ \hline
            E=16 & 129.55\begin{tiny}$\pm 0.17$\end{tiny}& 129.54\begin{tiny}$\pm 0.17$\end{tiny}& 125.64\begin{tiny}$\pm 0.24$\end{tiny}\\ \hline
            E=$\infty$ & 129.12\begin{tiny}$\pm 0.17$\end{tiny}& 127.01\begin{tiny}$\pm 0.18$\end{tiny}& 90.92\begin{tiny}$\pm 0.41$\end{tiny}
        \end{tabular}
    \end{minipage}
    \hspace{0.7in}
    \begin{minipage}{.35\linewidth}
        \includegraphics[width=\linewidth]{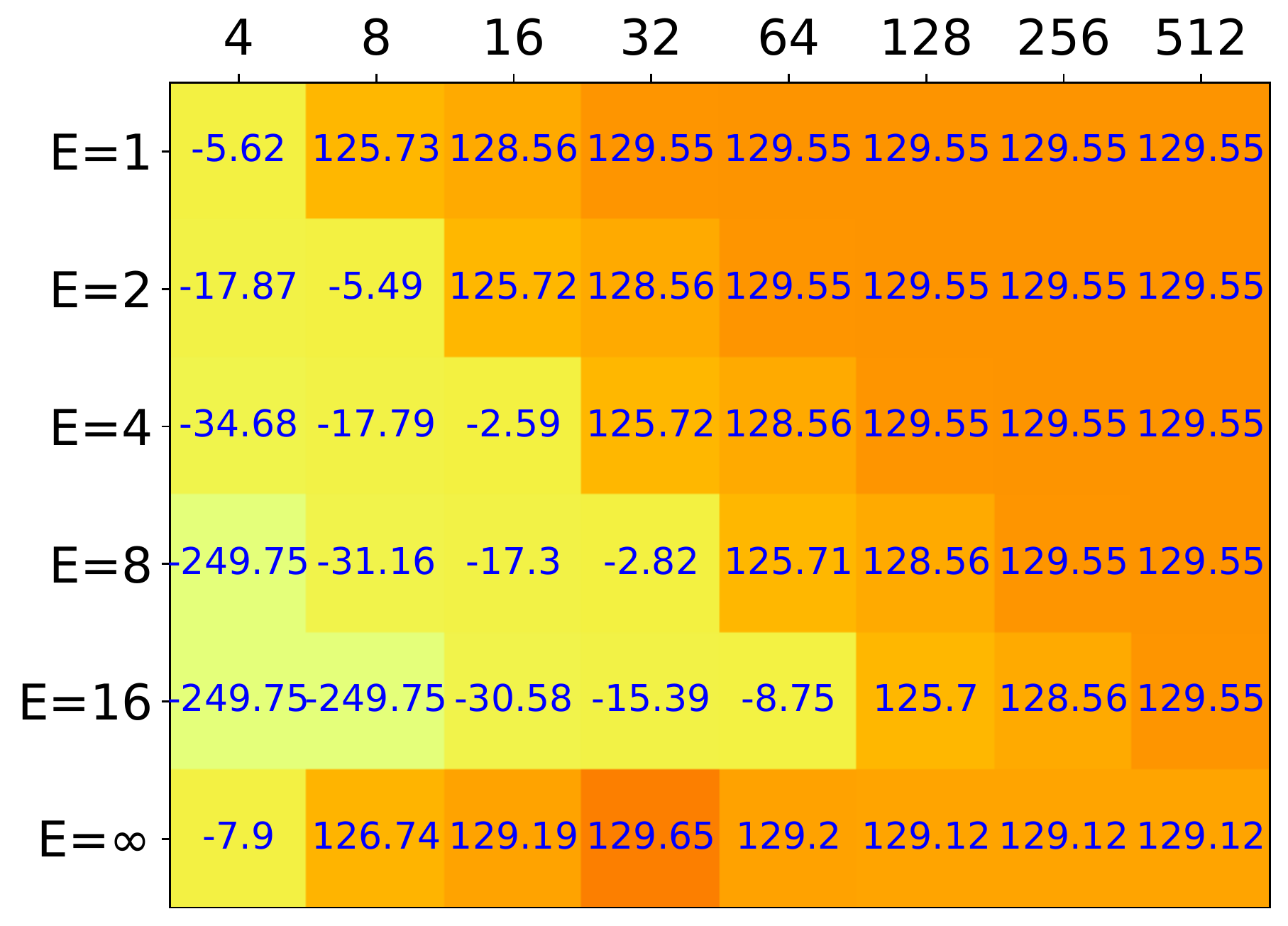}
    \end{minipage}
    \caption{Impact of local update time $E$ on convergent performance: larger $E$ indicates less frequent communication while $E=\infty$ means agents do not communicate; \textbf{Left} shows the objective values of FedRL at convergence. 
    \textbf{Right} shows the objective values of FedRL at different iterations during the training of \texttt{QAvg}s with different $E$.
    %\red{Shusen: What are the columns in the right figure?}}\blue{Hao: Columns indicate the iteration number (I have fixed the expression).
    }
    \label{tab&fig:windycliffs}
\end{figure*}

\section{Personalized FedRL}
\label{heuristics}

We propose a heuristic method that allows for better training in each local environment and better generalization to novel environments.
The idea is personalized FedRL, that is, instead of learning one policy for all the $n$ agents, we learn $n$ policies for the $n$ agents, respectively.
In this section, we consider deep FedRL; see Figure \ref{fig:peravg}.
We treat each environment as an ID and embed it into a low-dimension vector which is regarded as part of the state.
The $n$ agents share all the layers except the embedding layer.

After the training, the learned policy network may be applied to a never-seen-before environment.
In the new environment, the embedding layer cannot be reused. 
We need to let the agent interact with the new environment in order to learn the low-dimensional vector.
If the output of embedding is $d$-dimensional,  we need to learn only $d$ parameters.
Therefore, to generalize the trained policy network to a new environment, we need to perform few-shot learning in the new environment to learn the low-dimension embedding.

The benefit of the personalization heuristic is two-fold---better training and better generalization.
Without personalization, we seek to learn one policy that performs uniformly well in all the $n$ environments.
Since one policy cannot achieve the optimal performance in every environment, the learned policy is suboptimal in every environment.
With the $n$ private embedding layers, the convergent model serves as $n$ different policies for $n$ policies; each policy best fits one environment.
When the learned model is deployed to a never-seen-before environment, the few-shot learning of the embedding layer makes the policy quickly adapted to the new environment.
The small number of parameters to be tuned also adds robustness to the generalization process.

%% file: exper.tex
% \begin{figure}[t]
%     \centering
%     \includegraphics[width=0.9\linewidth]{fig/Tabular_train.pdf}
%     \caption{Training performance of \texttt{QAvg}s and \texttt{PAvg}s with different $E$ in tabular environments.}
%     \label{tabular_train}
% \end{figure}

\vspace{-0.1in}
\section{Empirical Study}

In this section, we firstly use tabular environments to verify our theories on \texttt{QAvg} and \texttt{PAvg}.
Then, we evaluate the extensions to deep reinforcement learning, \texttt{DQNAvg} and \texttt{DDPGAvg}, which are more practical in real-world applications.
Finally, we demonstrate that the personalization heuristic improves both training and generalization performance.
\footnote{Our code of both tabular cases and deep cases have been released on https://github.com/pengyang7881187/FedRL}

\subsection{Settings}
\label{exp_setting}
\textbf{Environments.} 
We construct a collection of heterogeneous environments by varying the state-transition parameters.
For example, given the \texttt{CartPole} environment, we vary the length of pole.
We use two types of tabular environments: first, random MDP with randomly generated state transition and reward function, and second, WindyCliff \cite{paul2019fingerprint} whose wind speed is uniformly sampled from $[S_{min},S_{max}]$.
We also use non-tabular environments in \texttt{Gym} \cite{1606.01540}: first, \texttt{CartPole} and \texttt{Acrobat} with varying length of pole, and second, \texttt{Hopper} and \texttt{Half-cheetah} with adjustable length of leg.

\textbf{Control.} 
For \texttt{QAvg}, after learning the averaged Q function $\bar{Q}(s,a)$, we use the deterministic policy, $\pi(s)=\argmax_{a\in\AM}\bar{Q}(s,a)$, for controlling the agent.
For \texttt{PAvg}, we directly learn a stochastic policy, $\pi(a|s)$, that outputs the probability of taking action $a$.
We use two types of \texttt{PAvg}: first, \texttt{ProjPAvg} denotes \texttt{PAvg} with projection operator, and second, \texttt{SoftPAvg} denotes \texttt{PAvg} with softmax activation function.

% \textbf{Convergent policy derivation.} For convergent models of these methods, we derive convergent policies in different ways.
% Specifically, for the Q-Learning based \texttt{QAvg}, we take the action corresponding to the maximal value of the convergent Q function at each state, \ie, $~\pi(s)=\texttt{argmax}_a \bar{Q}(s,a)$; for the policy gradient based \texttt{PAvg}, we take the policy corresponding to parameters of the convergent models.
% Moreover, in terms of \texttt{PAvg}, we evaluate two ways of implementation, softmax parameterization \cite{} (SoftPavg) and direct parameterization shown in theoretical analysis (ProjPAvg).

\textbf{Deep FedRL.} 
Deep Q Network (DQN) \cite{mnih2015human} and Deep Deterministic Policy Gradient (DDPG) \cite{lillicrap2015continuous} are two practical deep RL methods.
We extend our proposed \texttt{QAvg} and \texttt{PAvg} to DQN and DDPG; we call the extension \texttt{DQNAvg} and \texttt{DDPGAvg}.
Specifically, \texttt{DQNAvg} periodically approximately aggregates Q functions via averaging parameters of local Q networks, while \texttt{DDPGAvg} periodically aggregates both critic networks and policy networks stored in local devices.

\textbf{Baseline.}
The point of FedRL is to use all the agents' experience without directly sharing their experience.
As opposed to FedRL, independent RL lets each agent perform RL without exchanging information with other agents.
We use independent RL as the baseline for showing the usefulness of collaboration.
Let \texttt{Baseline} be the step-wise averaged objective values of the $n$ local models.
In other words, \texttt{Baseline} represents the performance of a randomly selected local model in $n$ involved environments.

\begin{figure}[t]
    \centering
    \includegraphics[width=0.9\linewidth]{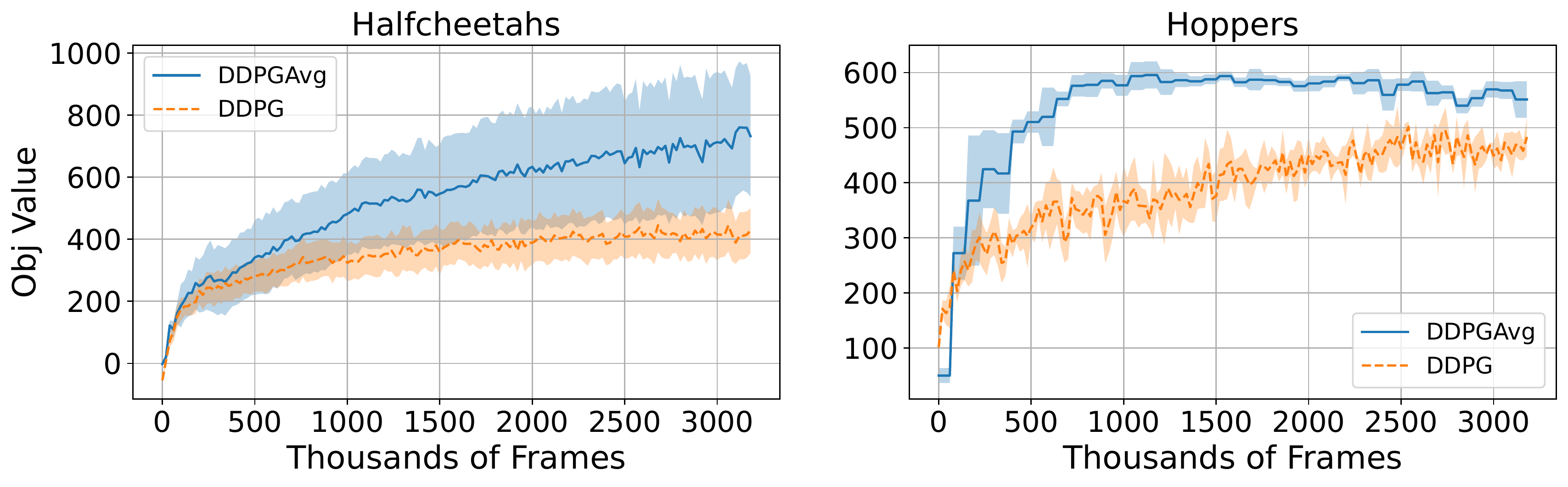}
    \caption{Acceleration of local training in federated setting: averaged local performance of locally trained policies is compared with averaged local performance of the policy trained in federated setting; we depict the mean as line and $1.65$ times of standard error as shadow.}
    \label{fed-speed}
\end{figure}
\begin{figure}[t]
    \centering
    \includegraphics[width=0.9\linewidth]{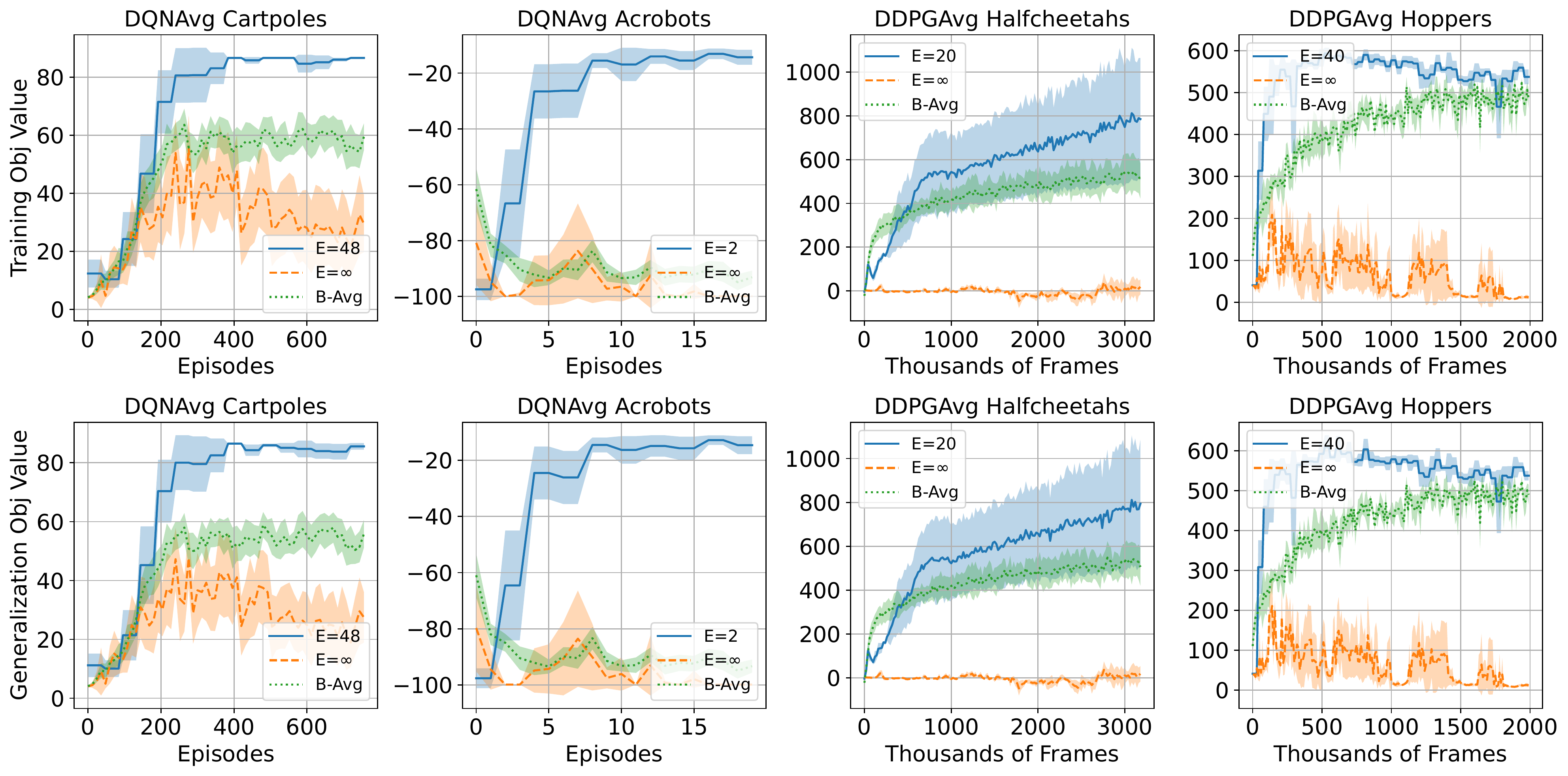}
    \caption{Training and generalization performance of DQNAvg and DDPGAvg in different tasks of FedRL: training performance refers to the objective value of FedRL, \ie, averaged performance in $N$ environments; generalization performance refers to the averaged performance in $M$ environments with newly generated state-transitions; we depict the mean as line and $1.65$ times of standard error as shadow.}
    \label{deep-exper}
\end{figure}
\vspace{-0.1in}

\subsection{Effect of Environment Heterogeneity}

To check the impact of environment heterogeneity on convergent performance of our methods, we construct tasks of FedRL with various $\kappa$, which controls how different the state transitions are.
Theorems \ref{qavg-converge} and \ref{pavg-converge} claim that larger environment heterogeneity, \ie, $\kappa$ with larger values, leads to larger performance gap with the optimal policy.
Empirical results shown in Table \ref{tab:env_heter} match such theoretical observations, and we discuss the experimental settings as below:

To get the control of environment heterogeneity with a scalar $\kappa$, we sample $N+1$ different state transitions $\{\PM_k\}_{k=0}^{N}$ and then construct the environments with $\{\PM_k^\kappa=\kappa \PM_k+(1-\kappa)\PM_0\}_{k=1}^N$.
$\{\PM_k^\kappa\}_{k=1}^N$ are $N$ copies of $\PM_0$ with noise, whose direction and intensity are respectively controlled with $\{\PM_k\}_{k=1}^N$ and $\kappa$.
With fixed $\{\PM_k\}_{k=0}^\kappa$, we manage to construct environments with environmental heterogeneity controlled by $\kappa$.
However, since the optimal policy for FedRL with $\{\PM_k^\kappa\}_{k=1}^N$ is computationally intractable, we make the following approximations: the convergent performance is approximated as the performance in noiseless central environment with $\PM_0$; the optimal policy is approximated as the optimal policy in the noiseless central environment $\PM_0$, \ie, the convergent policy with $\kappa=0$.
Each setting is repeated with $16,000$ random seeds.

\subsection{Effect of Communication Frequency}

We are next to show the impact of communication frequency on the convergence of our methods.
Specifically, the communication frequency is quantified by the number $E$ of local updates between two consecutive communications.
Empirical results in Figure \ref{tab&fig:windycliffs} reveal that communication frequency indeed influences \texttt{QAvg} and \texttt{PAvg}, yet quite in different ways.

For \texttt{QAvg}, Theorem \ref{qavg-converge} reveals that the convergent Q table is free of $E$ while communication frequency affects the convergence speed.
Figure \ref{tab&fig:windycliffs} \textbf{(Left)} confirms such a theoretical result with identical convergent values of \texttt{QAvg} with $E<\infty$, while Figure \ref{tab&fig:windycliffs} \textbf{(Right)} shows that \texttt{QAvg} with larger $E$ suffers from a lower convergence speed.
For \texttt{PAvg}, Theorem \ref{pavg-converge} reveals that the performance gap to the optimal policy is affected by $E$ and Remark \ref{pavgE} indicates the existence of optimal $E$.
Figure \ref{tab&fig:windycliffs} \textbf{(Left)} confirms that \texttt{PAvg}s, \texttt{ProjPAvg} and \texttt{SoftPAvg} have different convergent performance with different selection of $E$.
Moreover, performance peaks at $E=4,8$ respectively for \texttt{SoftPAvg} and \texttt{ProjPAvg} indicate that $E$ is a critical hyper-parameter in achieving the best convergent performance for \texttt{PAvg}s.

\begin{figure}[t]
    \centering
    \includegraphics[width=0.9\linewidth]{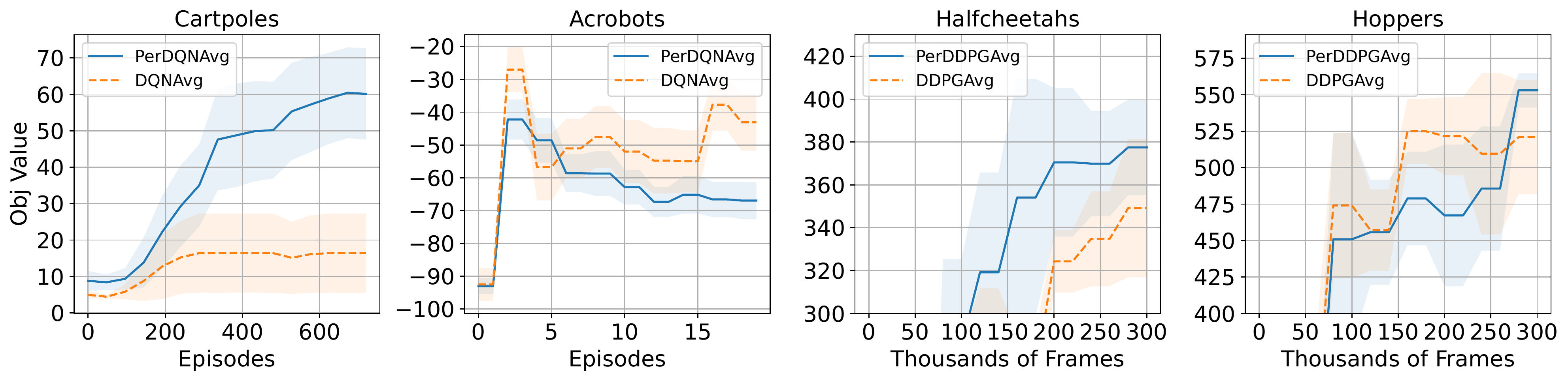}
    \caption{Improvement in local training with personalization heuristic: averaged local performance of \texttt{DQNAvg} and \texttt{DDPGAvg} is compared with the averaged local performance of \texttt{PerDQNAvg} and \texttt{PerDDPGAvg} with environment embeddings; we depict the mean as line and $1.65$ times of standard error as shadow.}
    \label{Person-local}
\end{figure}
\begin{figure}[t]
    \centering
    \includegraphics[width=0.9\linewidth]{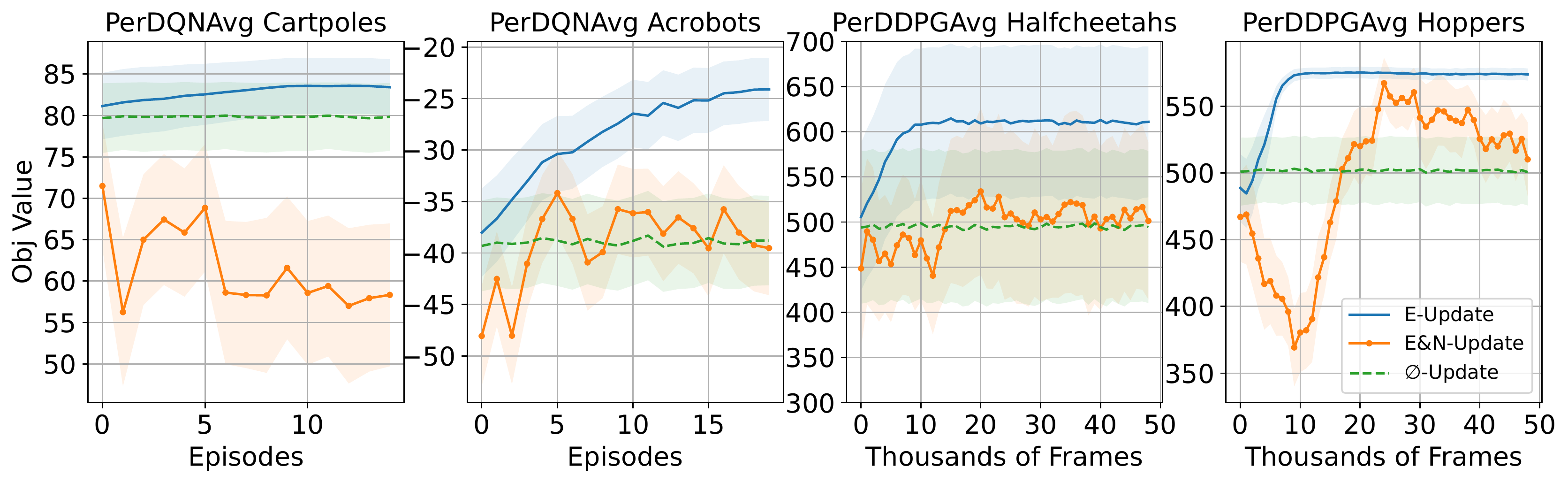}
    \caption{Impact of personalization heuristic on generalization performance: we compare different adjustment methods when fitting a novel environment given the learned convergent model; \texttt{E-Update} indicates only environment embeddings are adjusted, \texttt{E\&N-Update} indicates both embeddings and policy network are adjusted, and \texttt{$\Phi$-Update} keeps the learned model unchanged; we depict the mean as line and $1.65$ times of standard error as shadow.}
    \label{Person-general}
\end{figure}

\subsection{Experiments on Deep RL}
\label{pqavg}

Here we consider deep FedRL algorithms, \texttt{DQNAvg} and \texttt{DDPGAvg}, in more complicated FedRL tasks: \texttt{CartPoles} and \texttt{Acrobats} with discrete actions, \texttt{Halfcheetahs} and \texttt{Hoppers} with continuous actions.
In these scenarios, it is impossible to directly quantify environment heterogeneity $\kappa$ and we implicitly model it through sampling certain deciding parameters of state transitions from certain distribution. 
We first justify that the federated setting helps accelerate training in any individual environment, and then compare our methods with \texttt{Baseline} in terms of both training and generalization performance.

To figure out the impact of the federated setting on the training of any individual environment, we compare the averaged performance of local models in corresponding environments with and without communication with others.
Although collected experience is not allowed to share, communication of policy is believed to transfer certain knowledge from others.
As shown in Figure \ref{fed-speed}, policy communication indeed accelerates the local training and therefore alleviates the trouble of obtaining an efficient policy when data stored locally is limited.

Then we compare \texttt{DQNAvg} and \texttt{DDPGAvg} with their variants which do no communicate ($E=\infty$), and corresponding \texttt{Baseline}s, \ie, averaged performance of independently trained policies.
In terms of training performance, \texttt{DQNAvg} and \texttt{DDPGAvg} manage to obtain higher objective values of FedRL, which indicates the convergent policy uniformly performs well on all involved environments in FedRL.
Moreover, when faced with $M$ environments with newly generated state transitions, the learned policies of \texttt{DQNAvg} and \texttt{DDPGAvg} also outperform their variants with $E=\infty$ and \texttt{Baseline}s.
Therefore, the convergent policies of our methods not only efficiently solve the task of FedRL, but also generalize well to similar but unseen environments.

\subsection{Personalized FedRL}
We are now to demonstrate how the heuristic mentioned in Section \ref{heuristics} helps the personalization in the training of FedRL and how the learned personalized model enables us to quickly fit to any unseen environment.
\texttt{DQNAvg} and \texttt{DDPGAvg} with the heuristic are denoted as \texttt{PerDQNAvg} and \texttt{PerDDPGAvg}.

Figure \ref{Person-local} depicts the averaged performance of $N$ local policies in their corresponding environments with and without personalization heuristic.
Environment embeddings enable local policies to be personalized in the training process of \texttt{PerDQNAvg} and \texttt{PerDDPGAvg}, which helps to achieve better averaged local performance than the single aggregated policy of \texttt{DQNAvg} and \texttt{DDPGAvg}.
Moreover, when fitting the learned policy to any unseen environment, we merely adjust the environment embeddings from an averaged initialization.
Figure \ref{Person-general} reveals that such adjustment is enough for a quick fit to the novel environment and outperforms adjustment of both embeddings and policy network.

\section{Conclusion}

We have studied Federated Reinforcement Learning (FedRL) and addressed two issues: how to learn a single policy with uniformly good performance in all $n$ environments, and how to achieve personalization.
% Firstly, we propose algorithms, \texttt{QAvg} and \texttt{PAvg}, to learn one policy which performs uniformly well in all $n$ environments.
% Secondly, we propose heuristic for personalized FedRL which helps the local training in different environments and the generalization to any unseen environment.
The main difference from the existing FedRL work is that we assume that the $n$ environments have different state-transition functions.
We have  proposed two algorithms, \texttt{QAvg} and \texttt{PAvg}, which are federated extensions of Q-Learning and policy gradient.
Regarding their theoretical efficiency, we have analyzed their convergence and showed how environment heterogeneity affects the convergence.
We have also proposed a heuristic approach for personalization in FedRL, where environment embeddings are used to capture any specific environment.
Furthermore, such heuristic enables us to achieve generalization of convergent policies to fit any unseen environment via adjusting the embeddings.

\section*{Acknowledgments}
Jin and Zhang have been supported by the National Key Research and Development Project of China (No. 2018AAA0101004) and Beijing Natural Science Foundation (Z190001).

%% file: appendix.tex
\aistatstitle{Supplementary Materials}
\section{Proof of Theorem \ref{fedrl-example}}
We are next to find a task of FedRL having the following property: For some initial distribution $d_0$, $\forall \pi_1^\star \in \argmax_{\pi} g_{d_0} (\pi)$,
there exist $d_1$ and $\tilde{\pi}$ such that $g_{d_0}(\tilde{\pi})<g_{d_0}(\pi_1^\star)$, but $g_{d_1} (\tilde{\pi}) > g_{d_1} (\pi^\star_1 ) $. 

Consider the task of FedRL composed of the following two environments:
\begin{figure}[!htb]
\centering
\includegraphics[width=0.7\textwidth]{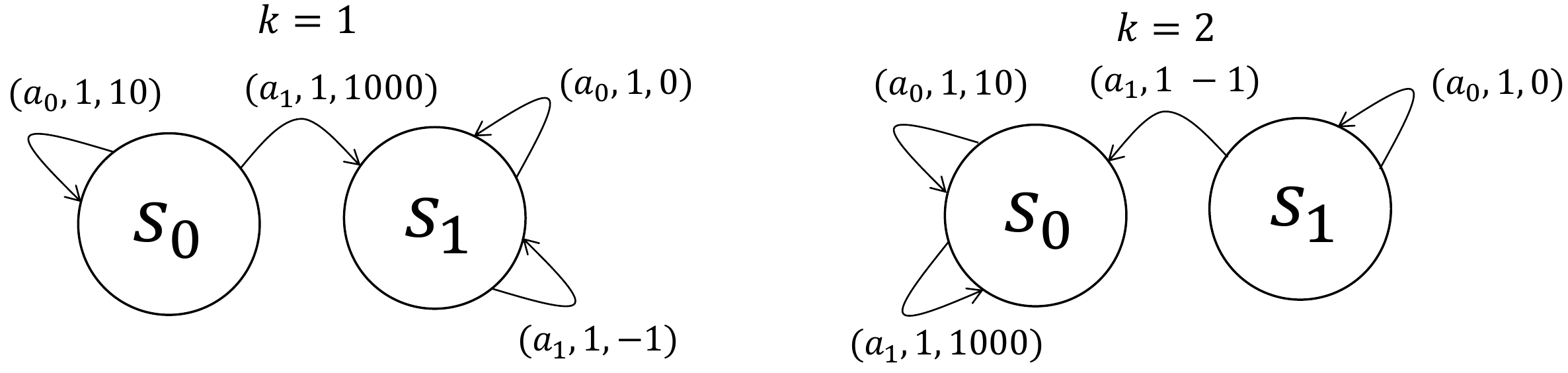}
\caption{Counterexample in FedRL: The triple means (action, probability, reward) and $\gamma=0.9$. Note that these two environments share the same action space $\{s_0,s_1\}$, same state space $\{a_0,a_1\}$, and same reward function.}
\label{Fig1}
\end{figure}
\vspace{-0.1in}
\begin{proof}
In the task of FedRL mentioned in Figure \ref{Fig1}, we use two real numbers $(p, q)\in [0, 1]^2$ to represent any policy $\pi$, where $p=\pi(a_0|s^0)$ and $q=\pi(a_0|s^1)$.
Let the initial state distribution be $d_0=(1,0)$, which means $s_0$ is initial state.
Therefore, the objective of FedRL is formulated as follows:
\begin{equation*}
    \max_{\pi} g_{d_0}(\pi)=\frac{1}{2}\{V_\pi^1(s_0)+V_\pi^2(s_0)\}.
\end{equation*}
It is a continuous function of policy $\pi=(p,q)$ whose support is compact.
Therefore, the optimal policy exists.
Such optimal policy $\pi^*=(p^*,q^*)$ is not unique, but we assert that $p^*<1, q^*=1$.

If $p^*=1$, then the cumulative rewards in FedRL is $\sum_{t=0}^\infty \gamma^t 10\approx 100$.
$p=0$ beats $p^*=1$ since $p=0$ earns $1000$ at the very first step in both environments.
In this way, we prove that $p^*<1$.
For the choices of $q^*$, if $q^*<1$, there is positive probability to take $a_1$ at $s_1$.
Since there is no probability to reach $s_1$ in the second environment when the initial state is $s_0$, we merely consider the first environment for the selection of $q^*$.
$q^*<1$ means there is positive probability to select $a_1$ at $s_1$.
Yet selection of $a_1$ at $s_1$ leads to negative reward $-1$, which is obviously inferior to the selection of $a_0$ whose reward is $0$ at $s_1$.
Therefore, we prove that $q^*=1$.

However, $\pi^*=(p^*,q^*)$ determined above is no longer the optimal solution when the initial state distribution changes to $d_1=(0,1)$.
When starting from $s_1$, $q^*=1$ means the agent never select $a_1$ at $s_1$ and the agent never reaches $s_0$ in the second environment.
Although selection of $a_1$ leads to a negative reward of $-2$ in both environments, yet the positive reward of actions at $s_0$ obviously compensates for the loss of choosing $a_1$ at $s_1$.
Therefore, $\pi^*=(p^*,q^*)$ above is no longer the optimal policy when the initial state distribution is formulated as $d_1=(0,1)$.
\end{proof}

\begin{figure}[!htb]
\centering
\includegraphics[width=0.7\textwidth]{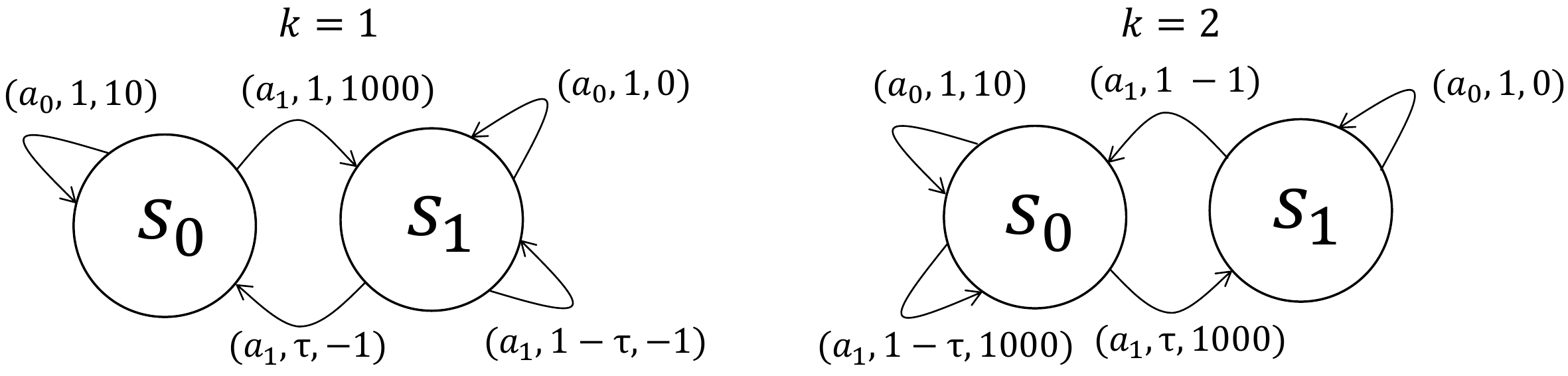}
\caption{The modified tasks of FedRL with two connected and irreducible environments.}
\label{Fig2}
\end{figure}

However, the above example is, to some extent, tricky, since both of the involved environments in FedRL are not irreducible.
Therefore, we propose the following example of FedRL with a positive lead probability $\tau>0$ in Figure \ref{Fig2}.
We claim that if leak probability $\tau$ is small enough, the previous argument still holds.
\begin{proof}
    In the task of FedRL mentioned in Figure \ref{Fig2}, we consider the following two initial state distributions $\{d_0=(1,0),d_1=(0,1)\}$.
    In this way, the objective functions in these two environments are denoted as $g_{d_0}(\pi)=\bar{V}_\pi^\tau(s_0)$ and $g_{d_1}(\pi)=\bar{V}_\pi^\tau(s_1)$, where $\tau$ is the leak probability and $\pi=(p,q)$.
    It is clear that both $g_{d_0}(\pi)$ and $g_{d_1}(\pi)$ are uniform continuous with respect to $(p,q,\tau)\in (0,1)^3$.
    We denote the set of optimal solutions w.r.t. these two initial state distributions as $\Gamma_0^\tau=\{(p_0^\tau,p_0^\tau)\}$ and $\Gamma_1^\tau=\{(p_1^\tau,p_1^\tau)\}$, and their objective values as $\tilde{M}_0^\tau$ and $\tilde{M}_0^\tau$.
    It is easy to see that both $\Gamma_0^\tau$ and $\Gamma_1^\tau$ are compact sets.
    
    Taking FedRL described in Figure \ref{Fig1} as a special case with $\tau=0$, we have already proved that $p_0^0<1,q_0^0=1,\forall (p_0^0,q_0^0)\in\Gamma_0^0$, and $q_1^0<1,\forall (p_1^0,q_1^0)\in\Gamma_1^0$.
    We claim that when $\tau$ is sufficiently small, there exists $\delta>0$, \texttt{s.t.} $q_0^\tau>1-\delta>a_1^\tau$.
    
    Firstly, we define $\alpha=\sum_{\Gamma_1^0}q_1^0$ with $\alpha<1$ by the compactness of $\Gamma_1^0$.
    Then we derive $M_1=\max_{q\geq\frac{1+\alpha}{2}}g_{d_1}^0(\pi)$ with $M_1\leq \tilde{M}_1^0$.
    The uniform continuity of objective values w.r.t. $\tau$ tells us: $\exists \bar{\tau}_1>0,~\forall \pi=(p,q),~\forall \tau<\bar{\tau}_1,~\texttt{s.t.}|g_{d_1}^\tau(\pi)-g_{d_1}^0(\pi)|<\frac{\tilde{M}_1^0-M_1}{4}$.
    Therefore, it is easy to tell $\forall \tau<\bar{\tau}_1,q_1^\tau<1-\delta$, where $\delta = \frac{1-\alpha}{2}$.
    
    Following the same strategy, we are able to derive that $\exists \bar{\tau}_2>0,\forall\tau<\bar{\tau}_2,q_0^\tau>1-\delta$, which along with $\forall \tau<\bar{\tau}_1,q_1^\tau<1-\delta$ leads to a contradiction.
\end{proof}

\section{Proof of Lemma\ref{lem:v:lowerbound}, \ref{lem:v:bound}}
\begin{proof}[Proof of Lemma 1]
The lower bound of weighted value function $\bar{V}_\pi$ is derived as:
\begin{align*}
    \bar{V}^\pi=&\frac{1}{n}\sum_{i=1}^n V^\pi_i = \frac{1}{n}\sum_{i=1}^n (I_{|\mathcal{S}|}-\gamma \PM^\pi_i)^{-1}R^\pi\\
    =&\frac{1}{n}\sum_{i=1}^n\sum_{k=0}^\infty (\gamma \PM^\pi_i)^k R^\pi\\
    \succcurlyeq&\sum_{k=0}^\infty(\gamma\frac{1}{n}\sum_{i=1}^n\PM^\pi_i)^k R^\pi=(I_{|\SM|}-\gamma\bar{\PM}^\pi)^{-1}R^\pi=V_I^\pi,
\end{align*}
where the second and fourth equalities come from $(I-A)^{-1}=\sum_{k=0}^\infty A^k$ with setting $A^0$ as $I$, and $\bar{V}^\pi \succcurlyeq V_I^\pi$ indicating $\bar{V}^\pi(s)\geq V_I^\pi(s)$.
\end{proof}
\begin{proof}[Proof of Lemma 2]
In fact, by definition of $\bar{V}^\pi$, for any $s\in\SM$, we have:
\begin{align*}
    \left|V^\pi_I(s)-\bar{V}^\pi(s)\right|&\le\frac{1}{n}\sum_{k=1}^n\left|V^\pi_I(s)-\bar{V}_k^\pi(s)\right|\\
    &\le\frac{1}{n}\sum_{k=1}^n\left\|V^\pi_I-\bar{V}_k^\pi\right\|_\infty
\end{align*}
By Bellman equation, we have:
\begin{align*}
    \left|V^\pi_I(s)-\bar{V}_k^\pi(s)\right|&=\gamma\left|\sum_{s'}\left(\frac{1}{n}\sum_{k=1}^n\PM_k^\pi(s'|s)V^\pi_I(s')-\PM_k^\pi(s'|s)V_k^\pi(s')\right)\right|\\
    &\le\frac{\kappa_1\gamma}{1-\gamma}+\gamma\|V^\pi_I-\bar{V}_k^\pi\|_\infty
\end{align*}
Thus, we have the final conclusion:
\begin{align*}
    \|V^\pi_I-\bar{V}^\pi\|_\infty\le\frac{\gamma\kappa_1}{(1-\gamma)^2}
\end{align*}
\end{proof}

\section{Proof of Theorem \ref{qavg-converge}}
Denote the Bellman Operator in the $k$-th environment as:
\begin{align*}
    \T_k Q(s,a) = R(s,a)+\gamma \sum_{s'}P_k(s'|s,a)\max_{a'}Q(s',a')
\end{align*}
The average Bellman Operator as:
\begin{align*}
    \T Q = \frac{1}{n}\sum_{k=1}^n \T_k Q
\end{align*}

\begin{thm}
    \label{thm: contraction}
    $\T$ is a $\gamma$-contractor. For any $Q_1$ and $Q_2$, it satisfies:
    \begin{align*}
        \|\T Q_1-\T Q_2\|_\infty\le\gamma\|Q_1-Q_2\|_\infty
    \end{align*}
\end{thm}
\begin{proof}
    By definition of $\T$, we have:
    \begin{align*}
        \|\T Q_1 - \T Q_2\|_\infty&\le\frac{1}{n}\sum_{k=1}^n \|\T_k Q_1 -\T_k Q_2\|_\infty\\
        &\le\gamma\|\T_k Q_1 -\T_k Q_2\|_\infty
    \end{align*}
\end{proof}

By Theorem~\ref{thm: contraction}, there exists a fixed point $Q^*$ satisfies $\T Q^*=Q^*$, which is also the optimal Q value function w.r.t. standard MDP with transition dynamics $\bar{\PM}$.

Besides, we also define a general version of average Bellman Operator:
\begin{align*}
    \T_E=\frac{1}{n}\sum_{k=1}^n \T_k^E
\end{align*}
and a smooth version:
\begin{align*}
    \widetilde{\T}_E=\frac{1}{n}\sum_{k=1}^n \prod_{t=1}^E(\lambda_t\T_k+(1-\lambda_t)Id)
\end{align*}

In other words, the update rule is:
\begin{align*}
    &Q_{t+1}^k = (1-\lambda_t) Q_{t}^k +\lambda_t\T_k Q_t^k\\
    &Q_{t+1}^k = \left\{
        \begin{aligned}
            &\frac{1}{n}\sum_{k=1}^n Q_{t+1}^k, \hspace{4pt}&\text{if $t+1\in\mathcal{I}_E$}\\
            &Q_{t+1}^k, \hspace{4pt}&\text{if $t+1\not\in\mathcal{I}_E$}
        \end{aligned}
    \right.
\end{align*}
We also denote $\bar{Q}_t=\frac{1}{n}\sum_{k=1}^n Q_{t}^k$.

\begin{lem}[One step recursion]
    \label{lem: Q_recur}
    We have:
    \begin{align*}
        \|\bar{Q}_{t+1}-Q^*\|_\infty\le(1-(1-\gamma)\lambda_t)\|\bar{Q}_{t}-Q^*\|_\infty+\frac{\lambda_t\gamma}{n}\sum_{k=1}^n\|Q_t^k-\bar{Q}_t\|_\infty
    \end{align*}
\end{lem}

\begin{proof}
    As $Q^*=\frac{1}{n}\sum_{k=1}^n \T_k Q^*$, we have:
    \begin{align*}
        \|\bar{Q}_{t+1}-Q^*\|_\infty=&\|(1-\lambda_t)\bar{Q}_t+\frac{\lambda_t}{n}\sum_{k=1}^n \T_k Q_t^k-Q^*\|_\infty\\
        =&\|(1-\lambda_t)(\bar{Q}_t-Q^*)+\frac{\lambda_t}{n}\sum_{k=1}^n (\T_k Q_t^k-\T_k Q^*)\|_\infty\\
        \le&(1-\lambda_t)\|\bar{Q}_t-Q^*\|_\infty+\frac{\lambda_t}{n}\sum_{k=1}^n \|\T_k Q_t^k-\T_k Q^*\|_\infty\\
        \le&(1-\lambda_t)\|\bar{Q}_t-Q^*\|_\infty+\frac{\gamma\lambda_t}{n}\sum_{k=1}^n \| Q_t^k-Q^*\|_\infty\\
        \le&(1-(1-\gamma)\lambda_t)\|\bar{Q}_t-Q^*\|_\infty+\frac{\gamma\lambda_t}{n}\sum_{k=1}^n \|Q_t^k-\bar{Q}_t\|_\infty
    \end{align*}
\end{proof}

\begin{lem}[Value variance]
    \label{lem: Q_var}
    Suppose $\lambda_t\le2\lambda_{t+E}$ and $Q_0^k\in[0,\frac{1}{1-\gamma}]$, we have:
    \begin{align*}
        \frac{1}{n}\sum_{k=1}^n \|Q_t^k-\bar{Q}_t\|_\infty\le \frac{4\lambda_t(E-1)}{(1-\gamma)}
    \end{align*}
\end{lem}

\begin{proof}
    Noting that $\mathbb{E}\|X-EX\|_\infty\le2\mathbb{E}\|X\|_\infty$, and for $\forall t$, there exists $t_0\le t$ and $t-t_0\le E-1$, such that $Q_{t_0}^k=\bar{Q}_{t_0}$. Thus we have:
    \begin{align*}
        \frac{1}{n}\sum_{k=1}^n \|Q_t^k-\bar{Q}_t\|_\infty&=\frac{1}{n}\sum_{k=1}^n \|Q_t^k-\bar{Q}_{t_0}+\bar{Q}_{t_0}-\bar{Q}_t\|_\infty\\
        &\le\frac{2}{n}\sum_{k=1}^n \|Q_t^k-\bar{Q}_{t_0}\|_\infty\\
        &=\frac{2}{n}\sum_{k=1}^n \|\sum_{t'=t_0}^{t-1}\lambda_{t'}\left(\T_k Q_{t'}^k - Q_{t'}^k\right)\|_\infty\\
        &\le\frac{2}{n}\sum_{k=1}^n \sum_{t'=t_0}^{t-1}\lambda_{t'}\|\T_k Q_{t'}^k - Q_{t'}^k\|_\infty\\
        &\le\frac{4\lambda_t(E-1)}{(1-\gamma)}
    \end{align*}
    where the last inequality holds by $Q_t^k\in[0,\frac{1}{1-\gamma}]$.
\end{proof}

\begin{proof}[Proof of Theorem \ref{qavg-converge}]
    By Lemma~\ref{lem: Q_recur} and Lemma~\ref{lem: Q_var}, we have:
    \begin{align*}
        \|\bar{Q}_{t+1}-Q^*\|_\infty\le(1-(1-\gamma)\lambda_t)\|\bar{Q}_t- Q^*\|_\infty+\frac{4\lambda_t^2\gamma(E-1)}{(1-\gamma)}
    \end{align*}
    To simplify, we denote $\Delta_{t+1}=\|\bar{Q}_{t+1}-Q^*\|_\infty$ and $C=\frac{4\gamma(E-1)}{(1-\gamma)}$, which leads to:
    \begin{align*}
        \Delta_{t+1}\le(1-(1-\gamma)\lambda_t)\Delta_t+\lambda_t^2\cdot C
    \end{align*}
    By setting $\lambda_t=\frac{\alpha}{t+\beta}$, we will prove $\Delta_t\le\frac{\zeta}{t+\beta}$ recursively:
    \begin{small}
    \begin{align*}
        \Delta_{t+1}&\le(1-(1-\gamma)\lambda_t)\frac{\zeta}{t+\beta}+\lambda_t^2\cdot C\\
        &=\frac{(t+\beta-1)\zeta}{(t+\beta)^2}+\frac{(1-(1-\gamma)\alpha)\zeta+\alpha^2\cdot C}{(t+\beta)^2}\\
        &\le\frac{\zeta}{t+\beta+1}
    \end{align*}
    \end{small}
    Trivially, we can set $\alpha=\frac{2}{1-\gamma}$ and $\zeta=\frac{4C}{(1-\gamma)^2}=\frac{16\gamma(E-1)}{(1-\gamma)^3}$. Besides, to satisfy $\lambda_t\le2\lambda_{t+E}$, we can set $\beta=E$. Thus, we have:
    \begin{align*}
        \|\bar{Q}_t-Q^*\|_\infty\le\frac{16\gamma(E-1)}{(1-\gamma)^3(t+E)}
    \end{align*}
\end{proof}

\section{Proof of Theorem \ref{pavg-converge}}
\begin{thm}[Global Optimality]
    Denote the optimal policy as $\pi^*$, for any  given policy $\pi\in\Delta(\AM)^{\SM}$, we have following inequality holds:
    \begin{align*}
        \frac{1}{n}\sum_{k=1}^{n}V^{\pi^*}_k(\mu)-\frac{1}{n}\sum_{k=1}^{n}V^{\pi}_k(\mu)\le 2(2L\eta+1)\rho\sqrt{|\SM|}\left(\kappa+\|G^\eta(\pi)\|_2\right)
    \end{align*}
    where
    \begin{align*}
        &G^\eta(\pi)=\frac{\proj\left(\pi+\eta\frac{1}{n}\sum_{k=1}^n \nabla_\pi V^{\pi}_{k}(\mu)\right)-\pi}{\eta}\\
        &G^\eta_k(\pi)=\frac{\proj\left(\pi+\eta \nabla_\pi V^{\pi}_{k}(\mu)\right)-\pi}{\eta}
    \end{align*}
\end{thm}

\begin{proof}
    By definition, we have:
    \begin{align*}
        \Delta(\pi)&=\frac{1}{n}\sum_{k=1}^{n}V^{\pi^*}_k(\mu)-\frac{1}{n}\sum_{k=1}^{n}V^{\pi}_k(\mu)\\
        &=\frac{1}{n}\sum_{k=1}^{n}\left(V^{\pi^*}_k(\mu)-V^{\pi}_k(\mu)\right)\\
        &=\frac{1}{n}\sum_{k=1}^{n}\frac{1}{1-\gamma}\mathbb{E}_{d_{\pi^*,\mu,k}}\langle\pi^*(\cdot|s), A^{\pi}_k(s,\cdot)\rangle\\
        &=\frac{1}{n}\sum_{k=1}^{n}\frac{1}{1-\gamma}\mathbb{E}_{d_{\pi^*,\mu,k}}\langle\pi^*(\cdot|s)-\pi(\cdot|s), A^{\pi}_k(s,\cdot)\rangle\\
        &=\frac{1}{n}\sum_{k=1}^{n}\frac{1}{1-\gamma}\mathbb{E}_{d_{\pi^*,\mu,k}}\langle\pi^*(\cdot|s)-\pi(\cdot|s), Q^{\pi}_k(s,\cdot)\rangle\\
        % &=\langle\pi^*-\pi,\frac{1}{(1-\gamma)n}\sum_{k=1}^{n}Q^\pi_k\odot d_{\pi^*,\mu,k}\rangle\\
        % &\le\langle\pi^*-\pi,\frac{\rho}{(1-\gamma)n}\sum_{k=1}^{n}Q^\pi_k\odot d_{\pi,\mu,k}\rangle\\
        % &\le\max_{\widetilde{\pi}}\langle\widetilde{\pi}-\pi,\frac{\rho}{(1-\gamma)n}\sum_{k=1}^{n}Q^\pi_k\odot d_{\pi,\mu,k}\rangle\\
        % &=\max_{\widetilde{\pi}}\langle\widetilde{\pi}-\pi,\frac{\rho}{n}\sum_{k=1}^{n}\nabla_\pi V^{\pi}_k(\mu)\rangle\\
        &\le\frac{1}{n}\sum_{k=1}^{n}\frac{1}{1-\gamma}\mathbb{E}_{d_{\pi^*,\mu,k}}\max_{\widetilde{\pi}}\langle\widetilde{\pi}(\cdot|s)-\pi(\cdot|s), Q^{\pi}_k(s,\cdot)\rangle\\
        &\le\frac{\rho}{n}\sum_{k=1}^{n} \max_{\widetilde{\pi}}\langle\widetilde{\pi}-\pi,\nabla_\pi V^\pi_k(\mu)\rangle\\
        &\le\frac{2\rho\sqrt{|\SM|}}{n}\sum_{k=1}^{n} \max_{\pi+\delta\in\Delta(\AM)^\SM, \|\delta\|_{2}\le1}\delta^T\nabla_\pi V^\pi_k(\mu)
    \end{align*}
    Denote $\pi_k^{+}=\pi+\eta G_k^\eta(\pi)$, we have:
    \begin{align*}
        &\max_{\pi+\delta\in\Delta(\AM)^\SM, \|\delta\|_{2}\le1}\delta^T\nabla_\pi V^\pi_k(\mu)\\
        &\le\left\|\nabla_\pi V^\pi_k(\mu)-\nabla_\pi V^{\pi_k^{+}}(\mu)\right\|_2+\max_{\pi+\delta\in\Delta(\AM)^\SM, \|\delta\|_{2}\le1}\delta^T\nabla_\pi V^{\pi_k^+}(\mu)\\
        &\le(2L\eta+1)\|G_k^\eta(\pi)\|_{2}
    \end{align*}
    Thus, we have:
    \begin{align*}
        \Delta(\pi)\le&\frac{2(2L\eta+1)\rho\sqrt{|\SM|}}{n}\sum_{k=1}^n \|G^\eta_k(\pi)\|_{2}\\
        \le&\frac{2(2L\eta+1)\rho\sqrt{|\SM|}}{n}\sum_{k=1}^n \|G^\eta_k(\pi)-G^\eta(\pi)\|_{2}+2(2L\eta+1)\rho\sqrt{|\SM|}\|G^\eta(\pi)\|_2\\
        \le&2(2L\eta+1)\rho\sqrt{|\SM|}\left(\kappa+\|G^\eta(\pi)\|_2\right)
    \end{align*}
\end{proof}

Denote the update rule as:
\begin{align*}
    &\pi_{t+1}^k=\pi_{t}^k+\eta_t G_k^{\eta_t}(\pi_t^k)\\
    &\pi_{t+1}^k = \left\{
        \begin{aligned}
            &\frac{1}{n}\sum_{k=1}^n \pi_{t+1}^k, \hspace{4pt}&\text{if $t+1\in\mathcal{I}_E$}\\
            &\pi_{t+1}^k, \hspace{4pt}&\text{if $t+1\not\in\mathcal{I}_E$}
        \end{aligned}
    \right.
\end{align*}
And we also denote $\bar{\pi}_t=\frac{1}{n}\sum_{k=1}^{n}\pi_{t}^k$ and $\bar{\pi}_{t+1}^+=\bar{\pi}+\eta_t G^{\eta_t}(\bar{\pi})$.

\begin{lem}[One step recursion]
    \label{lem: recursion}
    Measure the environment heterogeneity with $\kappa_2$, we have:
    \begin{align*}
        F(\bar{\pi}_{t+1})-F(\bar{\pi}_t)\ge& -\frac{\eta_t\kappa\sqrt{|\AM|}}{(1-\gamma)^2}-\frac{\eta_t L\sqrt{|\AM|}}{(1-\gamma)^2}\cdot\frac{1}{n}\sum_{k=1}^{n}\|\pi_t^k-\bar{\pi}_t\|_{2}+(\eta_t-\eta_t^2 L)\|G^{\eta_t}(\bar{\pi}_t)\|_2^2\\
        &-2\kappa^2\eta_t^2 L-\frac{2\eta_t^2L^3}{n^2}\left(\sum_{k=1}^n \|\pi_t^k-\bar{\pi}_t\|_2\right)^2 
    \end{align*}
\end{lem}

\begin{proof}
    WLOG, we denote $F_{k}(\pi)=V^{\pi}_k(\mu)$ and $F(\pi)=\frac{1}{n}\sum_{k=1}^n F_k(\pi)$. By L-smoothness, we have:
    \begin{align*}
        F(\bar{\pi}_{t+1})-F(\bar{\pi}_t)\ge\langle\nabla F(\bar{\pi}_t), \bar{\pi}_{t+1}-\bar{\pi}_{t}\rangle-\frac{L}{2}\|\bar{\pi}_{t+1}-\bar{\pi}_t\|_{2}^{2}
    \end{align*}
    Noting that $\bar{\pi}_{t+1}-\bar{\pi}_{t}=\eta_t\frac{1}{n}\sum_{k=1}^{n}G_k^{\eta_t}(\pi_{t}^k)$, we have:
    \begin{align*}
        F(\bar{\pi}_{t+1})-F(\bar{\pi}_t)\ge&\eta_t\langle\nabla F(\bar{\pi}_t), \frac{1}{n}\sum_{k=1}^{n}G_k^{\eta_t}(\pi_{t}^k)\rangle-\frac{\eta_t^2 L}{2}\|\frac{1}{n}\sum_{k=1}^{n}G_k^{\eta_t}(\pi_{t}^k)\|_{2}^{2}
    \end{align*}
    As $\bar{\pi}_{t+1}^+=\bar{\pi}_t+\eta_t G^{\eta_t}(\bar{\pi}_t)$ and the first order stationary condition, we have:
    \begin{align*}
        \langle\bar{\pi}_{t+1}^+-\bar{\pi}_t-\eta_t\nabla F(\bar{\pi}_t),\bar{\pi}_{t+1}^+-\bar{\pi}_t\rangle\le0
    \end{align*}
    which is equivalent with:
    \begin{align*}
        \langle G^{\eta_t}(\bar{\pi}_t)-\nabla F(\bar{\pi}_t), G^{\eta_t}(\bar{\pi}_t)\rangle\le0
    \end{align*}
    Therefore, we have:
    \begin{align*}
        F(\bar{\pi}_{t+1})-F(\bar{\pi}_t)\ge&\eta_t\langle\nabla F(\bar{\pi}_t), \frac{1}{n}\sum_{k=1}^{n}G_k^{\eta_t}(\pi_{t}^k)-G^{\eta_t}(\bar{\pi}_t)\rangle\\
        &+\eta_t\|G^{\eta_t}(\bar{\pi}_t)\|_{2}^{2}-\frac{\eta_t^2 L}{2}\|\frac{1}{n}\sum_{k=1}^{n}G_k^{\eta_t}(\pi_{t}^k)\|_{2}^{2}\\
        \ge&-\eta_t\|\nabla F(\bar{\pi}_t)\|_{2}\cdot\|\frac{1}{n}\sum_{k=1}^{n}G_k^{\eta_t}(\pi_{t}^k)-G^{\eta_t}(\bar{\pi}_t)\|_{2}\\
        &+\eta_t\|G^{\eta_t}(\bar{\pi}_t)\|_{2}^{2}-\frac{\eta_t^2 L}{2}\|\frac{1}{n}\sum_{k=1}^{n}G_k^{\eta_t}(\pi_{t}^k)\|_{2}^{2}
    \end{align*}
    Noting that:
    \begin{align*}
        \|\frac{1}{n}\sum_{k=1}^{n}G_k^{\eta_t}(\pi_{t}^k)-G^{\eta_t}(\bar{\pi}_t)\|_{2}\le&\|\frac{1}{n}\sum_{k=1}^{n}G_k^{\eta_t}(\pi_{t}^k)-G^{\eta_t}(\pi_t^k)\|_{2}\\
        &+\|\frac{1}{n}\sum_{k=1}^{n}G^{\eta_t}(\pi_{t}^k)-G^{\eta_t}(\bar{\pi}_t)\|_{2}\\
        \le&\kappa+\frac{L}{n}\sum_{k=1}^{n}\|\pi_t^{k}-\bar{\pi}_t\|_{2}
    \end{align*}
    Noting that $\|\nabla F\|_{2}\le\frac{\sqrt{|\AM|}}{(1-\gamma)^2}$, we have:
    \begin{align*}
        F(\bar{\pi}_{t+1})-F(\bar{\pi}_t)\ge& -\frac{\eta_t\kappa\sqrt{|\AM|}}{(1-\gamma)^2}-\frac{\eta_t L\sqrt{|\AM|}}{(1-\gamma)^2}\cdot\frac{1}{n}\sum_{k=1}^{n}\|\pi_t^k-\bar{\pi}_t\|_{2}\\
        &+\eta_t\|G^{\eta_t}(\bar{\pi}_t)\|_2^2-\frac{\eta_t^2 L}{2}\|\frac{1}{n}\sum_{k=1}^{n}G_k^{\eta_t}(\pi_{t}^k)\|_{2}^{2}
    \end{align*}
    Besides, by $\|a+b\|_{2}^2\le2\|a\|_2^2+2\|b\|_2^2$, we have:
    \begin{align*}
        \frac{1}{2}\|\frac{1}{n}\sum_{k=1}^{n}G_k^{\eta_t}(\pi_{t}^k)\|_{2}^{2}&\le\|\frac{1}{n}\sum_{k=1}^{n}G_k^{\eta_t}(\pi_{t}^k)-G^{\eta_t}(\bar{\pi}_t)\|_2^2+\|G^{\eta_t}(\bar{\pi}_t)\|_2^2\\
        &\le(\kappa+\frac{L}{n}\sum_{k=1}^{n}\|\pi_t^{k}-\bar{\pi}_t\|_{2})^2+\|G^{\eta_t}(\bar{\pi}_t)\|_2^2\\
        &\le2\kappa^2+\frac{2L^2}{n^2}\left(\sum_{k=1}^n \|\pi_t^k-\bar{\pi}_t\|_2\right)^2+\|G^{\eta_t}(\bar{\pi}_t)\|_2^2
    \end{align*}
    Gathering all these together, we have:
    \begin{align*}
        F(\bar{\pi}_{t+1})-F(\bar{\pi}_t)\ge& -\frac{\eta_t\kappa\sqrt{|\AM|}}{(1-\gamma)^2}-\frac{\eta_t L\sqrt{|\AM|}}{(1-\gamma)^2}\cdot\frac{1}{n}\sum_{k=1}^{n}\|\pi_t^k-\bar{\pi}_t\|_{2}\\
        &+(\eta_t-\eta_t^2 L)\|G^{\eta_t}(\bar{\pi}_t)\|_2^2\\
        &-2\kappa^2\eta_t^2 L-\frac{2\eta_t^2L^3}{n^2}\left(\sum_{k=1}^n \|\pi_t^k-\bar{\pi}_t\|_2\right)^2 
    \end{align*}
\end{proof}

\begin{lem}[Policy Variance]
    \label{lem: pivar}
    By $\|\nabla F_k\|_2\le\frac{\sqrt{\AM}}{(1-\gamma)^2}$ and assuming $\eta_t\le 2\eta_{t+E}$, we have:
    \begin{align*}
        \frac{1}{n}\sum_{k=1}^n \|\pi_t^k-\bar{\pi}_t\|_2^2\le \frac{4\eta_t^2(E-1)^2|\AM|}{(1-\gamma)^4}
    \end{align*}
\end{lem}

\begin{proof}
    For $\forall t$, there exists $t_0\le t$ and $t-t_0\le E-1$, such that $\pi_{t_0}^k=\bar{\pi}_{t_0}$. Thus we have:
    \begin{align*}
        \frac{1}{n}\sum_{k=1}^{n} \|\pi_t^k-\bar{\pi}_t\|_{2}^2&=\frac{1}{n}\sum_{k=1}^{n} \|\pi_t^k-\bar{\pi}_{t_0}+\bar{\pi}_{t_0}-\bar{\pi}_t\|_{2}^2\\
        &\le\frac{1}{n}\sum_{k=1}^n \|\pi_t^k-\bar{\pi}_{t_0}\|_2^2
    \end{align*}
    where the last inequality holds by $\mathbb{E}\|X-\mathbb{E}X\|_2^2\le\mathbb{E}\|X\|_{2}^2$. Noting that:
    \begin{align*}
        \pi_t^k = \pi_{t_0}^k+\sum_{t'=t_0}^{t-1}\eta_{t'} G_k^{\eta_{t'}}(\pi_{t'-1}^k)
    \end{align*}
    Thus, we have:
    \begin{align*}
        \frac{1}{n}\sum_{k=1}^{n} \|\pi_t^k-\bar{\pi}_t\|_{2}^2&\le\frac{1}{n}\sum_{k=1}^n \|\sum_{t'=t_0}^{t-1}\eta_{t'} G_k^{\eta_{t'}}(\pi_{t'-1}^k)\|_{2}^2\\
        &\le\frac{1}{n}\sum_{k=1}^n (t-t_0)\sum_{t'=t_0}^{t-1}\eta_{t'}^2\|G_k^{\eta_{t'}}(\pi_{t'-1}^k)\|_{2}^2\\
        &\le \frac{4\eta_t^2(E-1)^2|\AM|}{(1-\gamma)^4}
    \end{align*}
    where the last inequality holds by $\|G^\eta_k(\pi)\|_{2}\le\|\nabla F_k(\pi)\|_{2}$
\end{proof}

\begin{thm}[Full version of Theorem \ref{pavg-converge}]
    By setting $\eta_t=\sqrt{\frac{E}{12L^2(t+E/3)}}$, we have:
    \begin{align*}
        \min_{t=0,...,T-1}\|G^{\eta_t}(\bar{\pi}_t)\|_2^2\le&\frac{2\kappa\sqrt{|\AM|}}{(1-\gamma)^2}+\sqrt{\frac{24L^2}{E}}\cdot\frac{F(\bar{\pi}_T)-F(\bar{\pi}_0)}{\sqrt{T}}\\
        &+\sqrt{\frac{E}{6L^2}}\cdot\left(\frac{2 E(E-1)|\AM|L}{(1-\gamma)^4}+2\kappa^2 L\right)\cdot\frac{\log(1+\frac{3T}{E})}{\sqrt{T}}\\
        &=\widetilde{O}\left(\frac{\kappa\sqrt{|\AM|}}{(1-\gamma)^2}+\frac{|\AM|L}{\sqrt{T}}\right)
    \end{align*}
    where $\widetilde{O}(\cdot)$ omits logarithmic terms and some constants.
\end{thm}

\begin{proof}
    By Lemma~\ref{lem: recursion}, Lemma~\ref{lem: pivar} and $\eta_t\le\frac{1}{2L}$, we have:
    \begin{align*}
        F(\bar{\pi}_{t+1})-F(\bar{\pi}_t)\ge& -\frac{\eta_t\kappa\sqrt{|\AM|}}{(1-\gamma)^2}-\frac{\eta_t L\sqrt{|\AM|}}{(1-\gamma)^2}\cdot\frac{1}{n}\sum_{k=1}^{n}\|\pi_t^k-\bar{\pi}_t\|_{2}\\
        &+(\eta_t-\eta_t^2 L)\|G^{\eta_t}(\bar{\pi}_t)\|_2^2\\
        &-2\kappa^2\eta_t^2 L-\frac{2\eta_t^2L^3}{n^2}\left(\sum_{k=1}^n \|\pi_t^k-\bar{\pi}_t\|_2\right)^2 \\
        \ge&-\frac{\eta_t\kappa\sqrt{|\AM|}}{(1-\gamma)^2}-\frac{2\eta_t^2 (E-1) L|\AM|}{(1-\gamma)^4}\\
        &+\frac{\eta_t}{2}\|G^{\eta_t}(\bar{\pi}_t)\|_2^2-2\kappa^2\eta_t^2 L-\frac{8\eta_t^4(E-1)^2|\AM|L^3}{(1-\gamma)^4}\\
        \ge&-\frac{\eta_t\kappa\sqrt{|\AM|}}{(1-\gamma)^2}-\frac{2\eta_t^2 E(E-1)|\AM|L}{(1-\gamma)^4}\\
        &-2\kappa^2\eta_t^2 L+\frac{\eta_t}{2}\|G^{\eta_t}(\bar{\pi}_t)\|_2^2
    \end{align*}
    Summing over $t=0,1,...,T-1$, we have:
    \begin{align*}
        F(\bar{\pi}_{T})-F(\bar{\pi}_{0})\ge& -\frac{\kappa\sqrt{|\AM|}}{(1-\gamma)^2}\sum_{t=0}^{T-1}\eta_t-\frac{2 E(E-1)|\AM|L}{(1-\gamma)^4}\sum_{t=0}^{T-1}\eta_t^2\\
        &-2\kappa^2 L\sum_{t=0}^{T-1}\eta_t^2+\min_{t=0,...,T-1}\|G^{\eta_t}(\bar{\pi}_t)\|_2^2\sum_{t=0}^{T-1}\frac{\eta_t}{2}
    \end{align*}
    Re-aranging above inequality, we have:
    \begin{align*}
        \min_{t=0,...,T-1}\|G^{\eta_t}(\bar{\pi}_t)\|_2^2\le&\frac{2\kappa\sqrt{|\AM|}}{(1-\gamma)^2}+\frac{F(\bar{\pi}_{T})-F(\bar{\pi}_{0})}{\sum_{t=0}^{T-1}\eta_t}\\
        &+\left(\frac{2 E(E-1)|\AM|L}{(1-\gamma)^4}+2\kappa^2 L\right)\cdot\frac{\sum_{t=0}^{T-1}\eta_t^2}{\sum_{t=0}^{T-1}\eta_t}
    \end{align*}
    By setting $\eta_t=\sqrt{\frac{E}{12L^2(t+E/3)}}$, which satisfying $\eta_t\le\frac{1}{2L}$ and $\eta_t\le2\eta_{t+E}$, we have
    \begin{align*}
        \min_{t=0,...,T-1}\|G^{\eta_t}(\bar{\pi}_t)\|_2^2\le&\frac{2\kappa\sqrt{|\AM|}}{(1-\gamma)^2}+\sqrt{\frac{24L^2}{E}}\cdot\frac{F(\bar{\pi}_T)-F(\bar{\pi}_0)}{\sqrt{T}}\\
        &+\sqrt{\frac{E}{6L^2}}\cdot\left(\frac{2 E(E-1)|\AM|L}{(1-\gamma)^4}+2\kappa^2 L\right)\cdot\frac{\log(1+\frac{3T}{E})}{\sqrt{T}}
    \end{align*}
    where we use the following inequalities:
    \begin{align*}
        \sum_{t=0}^{T-1}\frac{1}{\sqrt{t+a}}&\ge2(\sqrt{T+a}-\sqrt{a})=\frac{2T}{\sqrt{T+a}+\sqrt{a}}\ge\sqrt{\frac{T}{2}},\\
        \sum_{t=0}^{T-1}\frac{1}{t+a}&\le\sum_{t=0}^{T-1}\log\left(1+\frac{1}{t+a}\right)=\log\frac{T+a}{a}.
    \end{align*}
\end{proof}

\section{Details of Empirical Results}
\subsection{Details of constructed environments.}
We construct tabular environments for ablation study on choices of $E$, and additionally modify several classical control tasks to evaluate deep methods and personalization heuristics.
\begin{itemize}
    \item \texttt{Random MDPs} is composed of $n$ environment, \texttt{Random MDP}.
    \texttt{Random MDPs} fix a randomly chosen reward function $R$ and generate a set of transition dynamics for each environment.
    Specifically, we set $N=5$ in the task of FedRL, and additionally sample $M=20$ transition dynamics (element-wisely Bernoulli distributed), \texttt{i.e.} 20 novel environments of \texttt{Random MDP} with same $R$, to test performance of generalization; we set $\gamma=0.9$; when testing the impact of environment heterogeneity, we evaluate \texttt{QAvg} and \texttt{SoftPAvg} with $E=4$, and \texttt{ProjPAvg} with $E=32$.
    
    \item \texttt{Windy Cliffs} is composed of $n$ environment, \texttt{Windy Cliff}.
    \texttt{Windy Cliff} is a modified version of a classic gridworld example from \cite{sutton1998introduction}: \texttt{Cliff Walking} environment. 
    The agent is expected to arrive the goal as fast as possible while avoiding falling off the cliff. 
    Just like the modified version considered in \cite{paul2019fingerprint}, we introduce a structured random noise in the environment, intensity $\theta$ of wind blowing from north.
    Specifically, $\theta$ is uniformly sampled from $U_{[0,1]}$, which means the agent could end up going down even if she does not intend to do that with a probability of $\frac{\theta}{3}$.
    In our setting, we experiment with the map of size $4\times 4$, and set the reward as $100$ and $-100$ for achieving the goal and falling off the cliff.
    Similarly, we set $n=5$ in the task of FedRL, and sample 20 novel environments of \texttt{Windy Cliff} to test performance of generalization; we set $\gamma=0.95$; when testing the impact of environment heterogeneity, we evaluate \texttt{QAvg} and \texttt{SoftPAvg} with $E=4$, and \texttt{ProjPAvg} with $E=32$.
    
    \item \texttt{CartPoles}: We construct \texttt{CartPoles} from \texttt{CartPole}.
    Different pole length indicates different pole mass, which leads to different state transition.
    Specifically, the pole length follows the uniform distribution $\mathcal{U}_{[0.2,1.8]}$.
    Additionally, we choose $N=5$, $M=20$ and $\gamma=0.99$ in the construction of FedRL.
    
    \item \texttt{Acrobats}: We construct \texttt{Acrobats} from \texttt{Acrobat}.
    Specifically, the mass of pole 1 follows the uniform distribution $\mathcal{U}_{[0.5,1.5]}$ when its pole length is fixed.
    Additionally, we choose $N=5$, $M=20$ and $\gamma=0.99$ in the construction of FedRL.
    
    \item \texttt{Halfcheetahs}: We construct \texttt{Halfcheetahs} from \texttt{Halfcheetah}.
    Specifically, the pole length of \texttt{bthigh} follows the uniform distribution $\mathcal{U}_{[0.1005,0.1855]}$, and the pole length of \texttt{fthigh} follows the uniform distribution $\mathcal{U}_{[0.1005,0.1655]}$.
    Additionally, we choose $N=5$, $M=20$ and $\gamma=0.99$ in the construction of FedRL.
    
    \item \texttt{Hoppers}: We construct \texttt{Hoppers} from \texttt{Hopper}.
    Specifically, the leg size follows the uniform distribution $\mathcal{U}_{[0.03,0.05]}$.
    Additionally, we choose $N=5$, $M=20$ and $\gamma=0.99$ in the construction of FedRL.
\end{itemize}